%% file: chen22a.tex
\documentclass[final,12pt]{alt2022} % Anonymized submission 
\usepackage{microtype}
\usepackage{graphicx}
\usepackage{booktabs} 
\usepackage{hyperref}
\usepackage{verbatim}
\usepackage{float}
\usepackage{amsfonts}
\usepackage{microtype}     
\usepackage{pifont}
\usepackage{nicefrac}     
\usepackage{mathtools}
\usepackage{wrapfig}
\usepackage{caption}
\usepackage{tikz,pgfplots}
\pgfplotsset{compat=1.11}
\usetikzlibrary{arrows.meta}
\mathtoolsset{showonlyrefs} 
\usepackage{algorithm}
\usepackage{algorithmic}
\newtheorem*{theorem*}{Theorem}

\include{symbol}

\usepackage{xcolor}

\title[]{Implicit Parameter-free Online Learning with Truncated Linear Models}
\usepackage{times}

\altauthor{%
 \Name{Keyi Chen} \Email{keyichen@bu.edu}\\
 \addr Department of Computer Science, Boston University, Boston, MA, USA\\
 \Name{Ashok Cutkosky} \Email{ashok@cutkosky.com}\\
 \Name{Francesco Orabona} \Email{francesco@orabona.com}\\
 \addr  Department of Electrical \& Computer Engineering, Boston University, Boston, MA, USA
}

\begin{document}

\maketitle

\begin{abstract}
Parameter-free algorithms are online learning algorithms that do not require setting learning rates.
They achieve optimal regret with respect to the distance between the initial point and \emph{any} competitor.
Yet, parameter-free algorithms do not take into account the geometry of the losses.
Recently, in the stochastic optimization literature, it has been proposed to instead use \emph{truncated} linear lower bounds, which produce better performance by more closely modeling the losses. In particular, truncated linear models greatly reduce the problem of overshooting the minimum of the loss function. Unfortunately, truncated linear models cannot be used with parameter-free algorithms because the updates become very expensive to compute.
In this paper, we propose new parameter-free algorithms that can take advantage of truncated linear models through a new update that has an ``implicit'' flavor. Based on a \emph{novel decomposition of the regret}, the new update is efficient, requires only one gradient at each step, never overshoots the minimum of the truncated model, and retains the favorable parameter-free properties. We also conduct an empirical study demonstrating the practical utility of our algorithms.
\end{abstract}


\input{intro}

\input{rel}

\input{def}

\input{difficulties}
\input{imp_coin}

\input{closed_form_sol}

\input{coor}

\input{exp}
\input{conc}

\acks{This material is based upon work supported by the National Science Foundation under the grants no. 1908111 ``AF: Small:
Collaborative Research: New Representations for Learning Algorithms and Secure Computation'' and no. 2046096 ``CAREER: Parameter-free Optimization Algorithms for Machine Learning''.}

\bibliography{../../../../learning}

\input{appendix}

\end{document}

%% file: symbol.tex
\newcommand{\bc}{\boldsymbol{c}}

\newcommand{\bg}{\boldsymbol{g}}
\newcommand{\bx}{\boldsymbol{x}}
\newcommand{\bu}{\boldsymbol{u}}
\newcommand{\by}{\boldsymbol{y}}

\newcommand{\bw}{\boldsymbol{w}}

\newcommand{\bv}{\boldsymbol{v}}

\newcommand{\bbeta}{\boldsymbol{\beta}}

\newcommand{\argmin}{\mathop{\mathrm{argmin}}}

\newcommand{\interior}{\mathop{\mathrm{int}}}
\newcommand{\dom}{\mathop{\mathrm{dom}}}

\newcommand{\field}[1]{\mathbb{#1}}

\newcommand{\R}{\field{R}}

\newcommand{\btheta}{\boldsymbol{\theta}}

\DeclareMathOperator{\Regret}{Regret}
\DeclareMathOperator{\Wealth}{Wealth}

%% file: intro.tex
\section{Introduction}

In this paper, we study Online Convex Optimization (OCO)~\citep{Gordon99,Zinkevich03}.
In this setting, for each of $T$ steps, a learner produces a prediction $\bw_t \in V$ in each step $t$, where $V \subseteq \R^d$ is the feasible convex set. After each prediction, an adversary reveals a convex loss function $\ell_t : V \to \R$ and the learner pays $\ell_t(\bw_t)$. The aim of the learner is to minimize its \emph{regret} with respect to any fixed prediction $\bu \in V$, defined as
\[
\Regret_T(\bu) \triangleq  \sum_{t=1}^T \ell_t(\bw_t) - \sum_{t=1}^T \ell_t(\bu)~.
\]

Depending on the assumptions on the feasible set and the losses, there are many OCO algorithms that achieve optimal regret.
The two main families of OCO algorithms are based on Online Mirror Descent (OMD)~\citep{NemirovskyY83,Warmuth97} and Follow-The-Regularized-Leader (FTRL)~\citep{Shalev-Shwartz07,AbernethyHR08,HazanK08}. For the particular case where $V\equiv \R^d$, \emph{parameter-free} algorithms are minimax optimal~\citep{OrabonaP16,CutkoskyO18}. The key feature of parameter-free algorithms is that the $\Regret_T(\bu)$ scales nearly linearly in $\|\bu\|$, and is \emph{constant} for $\bu=0$. This guarantee is only obtainable by popular strategies like Online Subgradient Descent~\citep{Zinkevich03} if the learning rate is carefully tuned to the (unknown!) value of $\|\bu\|$. This lack of tuning learning rates motivates the name ``parameter-free''. Yet, even this favorable minimax optimality might not be satisfactory.

In particular, most OCO algorithms simply approximate the losses using linear functions, ignoring their geometry. This approach is justified by the fact that the worst-case losses are indeed just linear ones. However, in the extremely common case that the losses are not actually linear, that is they are not worst-case, the algorithm is wasting potentially useful information.
More generally, too much focus on worst-case analyses and asymptotic rates can prevent the design of better algorithms.

In an effort to address this issue and go beyond focusing only on asymptotic rates, \citet{AsiD19} have proposed the use of \emph{truncated linear models} instead of linear models to obtain better stochastic Mirror Descent algorithms with negligible additional computational complexity. Truncated linear models are tighter lower bounds to the original function that do not require additional curvature while still yielding a closed form update for Mirror Descent algorithms. For example, a common issue with standard gradient descent methods is that they can \emph{overshoot} the minimum of the loss during any given iteration. The use of truncated linear models significantly mitigates this concern by providing a ``signal'' that gradient descent might overshoot, allowing the learning algorithm to take a more conservative step. The same idea can be applied to the online (rather than stochastic) setting, but only for OMD.  It is unknown how to use truncated linear models in parameter-free algorithms without having an explosion in the computational time.

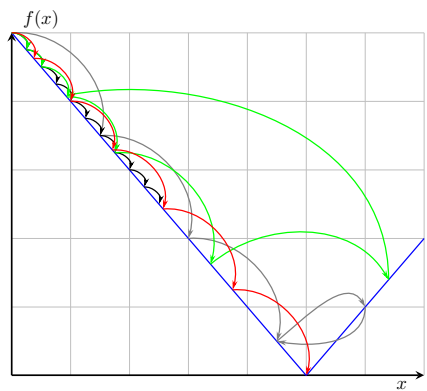
\begin{wrapfigure}{r}{0.4\textwidth}

\centering
\begin{tikzpicture}[scale=0.8, every node/.style={scale=0.8}]
\begin{axis}[axis line style = thick,
    domain = 0:14,
    samples = 200,
    axis x line = middle,
    axis y line = left,
    every axis x label/.style={at={(current axis.right of origin)},anchor=west,below = 2mm, right = 10mm},
    every axis y label/.style={at={(current axis.north west)},above=17mm,right = 1mm},
    xlabel = {$x$},
    ylabel = {$f(x)$},
    ticks = none,
    grid = major
    ]
    \addplot[blue,semithick] {abs(x-10)} [yshift=3pt] node[pos=.95,left]{};  
    \draw[-{Stealth[scale=0.8,angle'=30]},semithick, gray](axis cs:0,10) to [out=5,in=75] (axis cs:3,7){};
    \draw[-{Stealth[scale=0.8,angle'=30]},semithick, gray](axis cs:3,7) to [out=5,in=75] (axis cs:6,4){};
    \draw[-{Stealth[scale=0.8,angle'=30]},semithick, gray](axis cs:6,4) to [out=5,in=75] (axis cs:9,1){};
    \draw[-{Stealth[scale=0.8,angle'=30]},semithick, gray](axis cs:9,1) to [out=40,in=120] (axis cs:12,2){};
    \draw[-{Stealth[scale=0.8,angle'=30]},semithick, gray](axis cs:12,2) to [out=270,in=350] (axis cs:9,1){};    
    \draw[-{Stealth[scale=0.8,angle'=30]},semithick, black](axis cs:0,10) to [out=5,in=75] (axis cs:0.5,9.5){};
    \draw[-{Stealth[scale=0.8,angle'=30]},semithick, black](axis cs:0.5,9.5) to [out=5,in=75] (axis cs:1,9){};
    \draw[-{Stealth[scale=0.8,angle'=30]},semithick, black](axis cs:1,9) to [out=5,in=75] (axis cs:1.5,8.5){};
    \draw[-{Stealth[scale=0.8,angle'=30]},semithick, black](axis cs:1.5,8.5) to [out=5,in=75] (axis cs:2,8){};
    \draw[-{Stealth[scale=0.8,angle'=30]},semithick, black](axis cs:2,8) to [out=5,in=75] (axis cs:2.5,7.5){};
    \draw[-{Stealth[scale=0.8,angle'=30]},semithick, black](axis cs:2.5,7.5) to [out=5,in=75] (axis cs:3,7){};
    \draw[-{Stealth[scale=0.8,angle'=30]},semithick, black](axis cs:3,7) to [out=5,in=75] (axis cs:3.5,6.5){};
    \draw[-{Stealth[scale=0.8,angle'=30]},semithick, black](axis cs:3.5,6.5) to [out=5,in=75] (axis cs:4,6){};
    \draw[-{Stealth[scale=0.8,angle'=30]},semithick, black](axis cs:4,6) to [out=5,in=75] (axis cs:4.5,5.5){};
    \draw[-{Stealth[scale=0.8,angle'=30]},semithick, black](axis cs:4.5,5.5) to [out=5,in=75] (axis cs:5,5){};    
    \draw[-{Stealth[scale=0.8,angle'=30]},semithick, green](axis cs:0,10) to [out=5,in=75] (axis cs:0.5,9.5){};
    \draw[-{Stealth[scale=0.8,angle'=30]},semithick, green](axis cs:0.5,9.5) to [out=5,in=75] (axis cs:1,9){};
    \draw[-{Stealth[scale=0.8,angle'=30]},semithick, green](axis cs:1,9) to [out=5,in=75] (axis cs:15/8,10-15/8){};
    \draw[-{Stealth[scale=0.8,angle'=30]},semithick, green](axis cs:15/8,10-15/8) to [out=5,in=75] (axis cs:7/2,10-7/2){};
    \draw[-{Stealth[scale=0.8,angle'=30]},semithick, green](axis cs:7/2,10-7/2) to [out=5,in=75] (axis cs:27/4,10-27/4){};
    \draw[-{Stealth[scale=0.8,angle'=30]},semithick, green](axis cs:27/4,10-27/4) to [out=40,in=120] (axis cs:117/8*7/8,117/8*7/8-10){};
    \draw[-{Stealth[scale=0.8,angle'=30]},semithick, green](axis cs:117/8*7/8,117/8*7/8-10) to [out=90,in=10] (axis cs:117/64,10-117/64){};
    \draw[-{Stealth[scale=0.8,angle'=30]},semithick, red](axis cs:0,10) to [out=5,in=75] (axis cs:0.75,9.25){};
    \draw[-{Stealth[scale=0.8,angle'=30]},semithick, red](axis cs:0.75,9.25) to [out=5,in=75] (axis cs:2,8){};
    \draw[-{Stealth[scale=0.8,angle'=30]},semithick, red](axis cs:2,8) to [out=5,in=75] (axis cs:3.42,6.58){};
    \draw[-{Stealth[scale=0.8,angle'=30]},semithick, red](axis cs:3.42,6.58) to [out=5,in=75] (axis cs:5.14,4.86){};
    \draw[-{Stealth[scale=0.8,angle'=30]},semithick, red](axis cs:5.14,4.86) to [out=5,in=75] (axis cs:7.5,2.5){};
    \draw[-{Stealth[scale=0.8,angle'=30]},semithick, red](axis cs:7.5,2.5) to [out=5,in=75] (axis cs:10,0){};   
    \end{axis}
\end{tikzpicture}

\caption{Coin-Betting with truncated linear models (Red), Coin-Betting with linear models (Green), OGD with large constant stepsizes (Grey), and OGD with small constant stepsizes (Black). }
\label{fig:illustration}

\end{wrapfigure}

In this work, we propose new parameter-free algorithms that are able to take advantage of truncated linear models. Note that any optimization algorithm based on linear models can overshoot the optimum, but parameter-free algorithms may be even more prone to overshooting because their iterates can move \emph{exponentially far} between iterations. Instead, our new algorithms effectively alleviate this problem, see Figure~\ref{fig:illustration}. Our algorithms are based on a \emph{new decomposition of the regret that takes advantage of the geometry of truncated linear losses} that might be of independent interest.

In summary, our primary contribution is a new algorithm that maintains optimal parameter-free regret bounds \emph{but also} incorporates additional geometric information about the loss functions. While such an improvement is not visible in \emph{worst-case} rates, we demonstrate through an ``implicit-style'' regret bound that the algorithm could perform \emph{significantly better} in practice, and verify this behavior in an empirical study.

The rest of the paper is organized as follows: in Section~\ref{sec:rel} we discuss related work and in Section~\ref{sec:def} we review some definitions and background knowledge. In Section~\ref{sec:difficulty}, we show the difficulties in using truncated linear models in parameter-free algorithms. In Section~\ref{sec:imp_coin}, we present our solution and prove a bound on its regret. Since this algorithm does not have a closed form update rule, in Section~\ref{sec:closed_form_sol} we propose a more efficiently computable variant while still retaining the same theoretical guarantee. In Section~\ref{sec:coor}, we present a coordinate-wise extension that obtains a tighter bound as well as better empirical performance. Finally, in Section~\ref{sec:exp}, we empirically validate our algorithm.

%% file: rel.tex
\section{Related work}
\label{sec:rel}

\noindent\textbf{Parameter-free OCO Algorithms}
Parameter-free OCO algorithms are motivated by a desire to avoid choosing a step size and can achieve optimal theoretical regret bounds~\citep[e.g.,][]{McMahanO14,Orabona14,OrabonaP16,CutkoskyB17, FosterRS18,CutkoskyO18,Kotlowski19,KempkaKW19,CutkoskyS19,JunO19, vanderHoeven19, MhammediK20, OrabonaP21, ChenLW21}. Some of them are based on the FTRL framework~\citep{Shalev-Shwartz07,AbernethyHR08,HazanK08} (sometimes indirectly through methods such as coin-betting). The closest work to our algorithms is the CODE algorithm~\citep{ChenLO22} which is the first attempt to combine parameter-free methods with truncated losses. Inspired by the Importance Weight Awareness updates in~\citet{KarampatziakisL11}, CODE models the optimization algorithm with an ODE, and solves the ODE in a closed form to make infinitely many infinitesimal parameter-free updates on truncated losses. While CODE solves the ODE in closed form, it does not have any theoretical guarantee. In our work, in each step $t$, we consider the loss in two points only: on the current prediction and the updated one. This gives rise to an implicit equation that we can solve for truncated losses and to an optimal regret guarantee.

\noindent\textbf{Truncated Linear Models and Implicit Updates}
Truncated linear models were proposed in \citet{AsiD19} to create a tighter surrogate model for optimization.
While the use of convex linear lower bounds is also the core method in OCO algorithms~\citep[see, e.g.,][]{Orabona19}, 
we are not aware of any other online learning algorithm with a regret guarantee based on truncated linear models.
\citet{AsiD19} incorporate truncated linear models into the Mirror Descent update~\citep{NemirovskyY83}, forming a
\emph{proximal/implicit} update~\citep{Moreau65,Martinet70,Rockafellar76,KivinenW97,ParikhB14}.
In online learning, \citet{KulisB10} provides the first regret bounds for implicit updates that
match those of OMD, while \citet{McMahan10} makes the first attempt to quantify the advantage of the implicit updates in the
regret bound. \citet{SongLLJZ18} generalize the results in \citet{McMahan10} to Bregman divergences and strongly convex
functions, and quantify the gain differently in the regret bound. Finally, \citet{CampolongoO20} show that implicit
updates give rise to regret guarantees that depend on the temporal variability of the losses as well.
We will match the dependency on the subgradients in our final results to the one of FTRL with implicit
updates~\citep{McMahan10}, which underlines the ``implicit'' nature of our algorithm.

%% file: def.tex
\section{Preliminary}
\label{sec:def}

In this section, we introduce some of the needed background and definition. 

\noindent \textbf{Convex Analysis Definitions}
For a function $f:\R^d \to \R \cup \{+\infty\}$, we define a subgradient of $f$ in $\bx \in \R^d$ as a vector $\bg \in \R^d$ that satisfies $f(\by)\geq f(\bx) + \langle \bg, \by-\bx\rangle, \ \forall \by \in \R^d$.
We denote the set of subgradients of $f$ at $\bx$ by $\partial f(\bx)$.
A function $f : \R^d \to \R \cup \{+\infty\}$ is $\mu$-strongly convex over a convex set $V \subseteq \interior \dom f$ w.r.t. $\|\cdot\|$ if $\forall \bx, \by \in V$, we have $\bg \in \partial f(\bx)$,  $f(\by) \geq f(\bx) + \langle \bg , \by - \bx \rangle + \frac{\mu}{2} \| \bx - \by \|^2$. 
The Fenchel conjugate $f^\star$ of a function $f:\R^d\to\R$ is defined as $f^\star (\btheta) = \sup_{\bx} \langle \btheta, \bx\rangle - f(\bx)$. We denote the projection of a vector $\bx$ onto a convex set $B$ as $\Pi_B(\bx)\triangleq\argmin_{\bu \in B} \ \|\bx-\bu\|^2$.

\noindent\textbf{Coin-Betting and Online Learning}
We now explain the coin-betting framework for parameter-free algorithm design~\citep{OrabonaP16}, which operates through convex duality. We consider a vector-valued ``coin'' $\bc_t\in \R^d$ with $\|\bc_t\|\le 1$ provided to a gambler in response to a ``bet'' $\bx_t\in \R^d$. The gambler earns $\langle \bc_t, \bx_t\rangle$ dollars, for a total wealth of $\Wealth_t = \epsilon + \sum_{i=1}^t \langle \bc_i, \bx_i\rangle$ at time $t$, assuming an initial endowment of $\epsilon$. We enforce that $\bx_t = \bbeta_t \Wealth_{t-1}$ for some \emph{betting fraction} $\|\bbeta_t\|\le 1$, which intuitively corresponds to preventing the gambler from betting more money than the gambler has: $\Wealth_t\ge 0$ for all $t$. The goal of the gambler is of course to make the wealth as high as possible. To use this gambling game in online learning, set $\bc_t \in -\partial \ell_t(\bx_t)$, and let the learner's $\bw_t\in \R^d$ be simply equal to the gambler's $\bx_t$. To analyze the regret, suppose that $\Wealth_T\ge H\left(\sum_{t=1}^T \bc_t\right)$ for some arbitrary function $H$. Then we have:
\begin{align*}
    \sum_{t=1}^T \ell_t(\bw_t) - \ell_t(\bu)&\le \sum_{t=1}^T \langle \bc_t,  \bu - \bw_t \rangle =\epsilon+\sum_{t=1}^T \langle \bc_t,\bu\rangle - \Wealth_T\\
    &\le \epsilon+\sum_{t=1}^T \langle \bc_t,\bu\rangle  - H\left(\sum_{t=1}^T \bc_t\right)\le \epsilon+\sup_{G\in \R^d} \langle G,\bu\rangle - H(G)=\epsilon+H^\star(\bu)
\end{align*}
where in the first inequality we use the definition of the subgradient, in the second the assumption on $H$, and the last equality the definition of Fenchel conjugate $H^\star$.

Critically, notice that the wealth lower-bound $\Wealth_T\ge H\left(\sum_{t=1}^T \bc_t\right)$ does not involve $\bu$. Instead, $\bu$ appears only in analysis through Fenchel duality, which provides the parameter-free property. Hence, \emph{we can use any betting algorithm that guarantees a high wealth to design a parameter-free optimization algorithm.} 

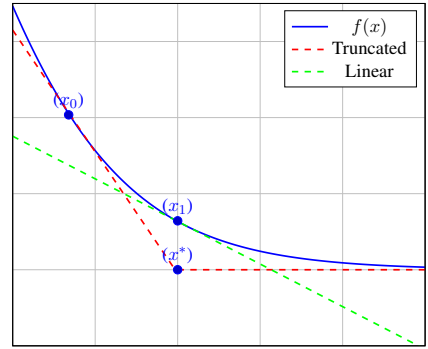
\begin{wrapfigure}{r}{0.4\textwidth}

\centering
\begin{tikzpicture}[scale=0.8, every node/.style={scale=0.8}]
\begin{axis}[
    domain = 0:5,
    samples = 200,
    ticks = none,
    grid = major,
    xmin=0, xmax=5,
    ymin=-0.2, ymax=0.7,
    ]
    \addplot[blue,thick] {ln(1+e^(-x))} [yshift=0pt] node[pos=.95,left]{}; 
    \addplot[red,thick,dashed] {max(0.47-0.32*(x-0.5),0)} [yshift=0pt] node[pos=.95,left]{};
    \addplot[green,thick,dashed] {0.127-0.112*(x-2)} [yshift=0pt] node[pos=.95,left]{};
    \legend{$f(x)$,Truncated,Linear}
    \addplot+[only marks,forget plot,blue] coordinates {(2,0)};
\end{axis}

\begin{axis}[nodes near coords,enlargelimits=0.2,
    domain = 0:5,
    samples = 200,
    ticks = none,
    grid = none,
    xmin=0, xmax=5,
    ymin=-0.2, ymax=0.7,
    ]
	\addplot+[only marks,
		point meta=explicit symbolic] 
	coordinates {
		(-0.05,0.47) [($x_0$)]
		(1.8,0.08 ) [($x_1$)]
		(1.8,-0.1) [($x^{*}$)]
	};
\end{axis}
\end{tikzpicture}
\caption{Models of the function $f(x)=\log (1+e^{-x})$: a truncated linear model (Red) built around the point $x_0$, and a linear model (Green) built around the point $x_1$. $x^{*}$ is the hinge corner.}
\label{fig:trunc}

\end{wrapfigure}

\noindent\textbf{Truncated Linear model}
For any $\hat{\ell}_t$ such that $\hat{\ell}_t(\bw_t) = \ell_t(\bw_t)$ and $\hat{\ell}_t(\bw) \leq \ell_t(\bw),~ \forall \bw$, we have $\sum_{t=1}^T (\ell_t(\bw_t) - \ell_t(\bu)) 
\leq \sum_{t=1}^T (\hat{\ell}_t(\bw_t) - \hat{\ell}_t(\bu))$. Linear models $\hat{\ell}_t(\bw) = \ell_t(\bw_t) + \langle \bg_t, \bw - \bw_t\rangle$ satisfy this property with $\bg_t \in \partial \ell_t(\bw_t)$, which motivates the popularity of online \emph{linear} optimization. However, we might think to design tighter approximations. In particular, \citet{AsiD19} proposed truncated linear models: 
\[
\hat{\ell}_t(\bw) \triangleq \max\{\ell_t(\bw_t)+\langle\bg_t, \bw -\bw_t\rangle, \inf_{\bw} \ell_t(\bw)\}~.
\] 
Besides the property above, truncated linear models satisfy: i) $\hat{\ell}_t(\bw)$ is convex and subdifferentiable on the domain; ii) For any $\bw$, we have $\bg^+=h \bg_t$ where $\bg^+ \in \partial \hat{\ell}_t(\bw)$ and $h \in [0,1]$;
iii) $\hat{\ell}_t(\bw) \geq \inf_{\bw} \ell_t(\bw)$.
In the following, we will assume that $\inf_{\bw} \ell_t(\bw)=0$, w.l.o.g. for loss functions bounded from below.

%% file: difficulties.tex
\section{Difficulties in Using Truncated Models in Parameter-free Algorithms}
\label{sec:difficulty}

Many parameter-free algorithms are based on FTRL. Hence, it is natural to ask whether it is possible to directly use truncated linear models instead of linear models in FTRL to utilize truncated linear models in parameter-free algorithms.
This approach immediately runs into significant problems. Specifically, FTRL algorithms usually maintain the sum of the losses observed so far.

One can easily store this sum when the losses are linear, but truncated losses would require $\mathcal{O}(T)$ space and $\mathcal{O}(poly(T))$ time for every update.
Thus, using truncated linear models with FTRL has the same computational cost as using the original cost functions - 
the simplification to using truncated linear models does not appear to help.
In contrast, our solution has the same computational and space complexity of online gradient descent.

Another possibility is to adapt the coin-betting design of parameter-free algorithms~\citep{OrabonaP16}
to truncated linear models. A moment of thinking should convince the reader this is far from simple:
The reduction from optimization to coin-betting described in
Section~\ref{sec:def} works by transforming subgradients into coin outcomes, but
the subgradient of a truncated loss is exactly the same as the subgradient of the original function!
Thus, simply using the reduction as-is on the truncated linear model would provide
\emph{no benefit} over using the simpler linear model.

Another method that seems possible is using truncated linear models in online gradient descent,
and then make online gradient descent parameter-free via some application of
the doubling trick~\citep[2.3.1]{Shalev-Shwartz12}. 
Perhaps surprisingly, employing the doubling trick in this way is quite difficult.
The only known application of the doubling trick is in \citet{StreeterM12}, but it
does not achieve the optimal regret bound, and more importantly, employs a delicate identity relating the regret
and sums of gradients that may fail for the truncated linear model.
Furthermore, the doubling trick usually has terrible empirical performance,
which completely defeats the purpose of using truncated linear models.

%% file: imp_coin.tex
\section{Parameter-free OCO with Truncated Linear Models}
\label{sec:imp_coin}

In this section, we introduce our novel parameter-free algorithms for truncated linear models.

We overcome the difficulties described above through a multi-step process, during which we will introduce three separate algorithms.
First, we introduce a new regret decomposition to take advantage of truncated losses while still only requiring storage of a few vectors. 
As an illustration of the key principles, we will use this new regret decomposition to design Algorithm~\ref{alg:imp_coin}.
While the theoretical guarantee of Algorithm~\ref{alg:imp_coin} matches our desiderata, the update does not have a closed form.
Hence, we then show how to slightly change our algorithm to obtain a closed form update in Algorithm~\ref{alg:imp_coin_closed_form}. Finally, in Section~\ref{sec:coor}, we consider each coordinate as a separate 1-d problem to obtain a coordinate-wise variant that achieves better performance both theoretically and empirically. 

As mentioned previously, using truncated linear models in an FTRL-based parameter-free algorithm would result in an inefficient update.

Therefore, in the following, we show a different approach inspired by the idea of implicit updates~\citep{KivinenWH06,KulisB10,McMahan10}. Our method introduces a new variation on the standard approach to bounding online convex optimization with online linear optimization, and an accompanying update to the regret/reward duality.

\noindent\textbf{A New Regret Decomposition}
We are interested in upper bounding the terms $\ell_t(\bw_t) - \ell_t(\bu)$ for any $\bu \in \R^d$. The usual method~\citep{Zinkevich03} is to upper bound the regret by linear terms, and then proceed to bound the regret on the linear losses as follows:
\[\Regret_T(\bu) 
=\sum_{t=1}^T \ell_t(\bw_t) - \ell_t(\bu) 
\leq \sum_{t=1}^T \langle \bg_t, \bw_t - \bu \rangle,
%&= \left\langle \sum_{t=1}^T -\bg_t,\bu \right\rangle -\sum_{t=1}^T\langle -\bg_t,\bw_t\rangle  ~.
\]
where $\bg_t \in \partial \ell_t(\bw_t)$. While this approach gives worst-case optimal upper bounds, it completely ignores the geometry of the loss functions $\ell_t$. In contrast, we consider upper bounding the term $\ell_t(\bw_t) - \ell_t(\bu)$ with the truncated linear loss $\hat{\ell}_t$, and decompose the regret on the truncated linear losses from an ``implicit'' point of view for a tighter bound. Specifically, for any $\bu \in\R^d$ we have
\begin{equation}
\label{eq:decomposition}
\begin{aligned}
\ell_t(\bw_t) - \ell_t(\bu) &\leq  \hat{\ell}_t(\bw_t) - \hat{\ell}_t(\bu) =\hat{\ell}_t(\bw_t) - \hat{\ell}_t(\bw_{t+1})+\hat{\ell}_t(\bw_{t+1})-\hat{\ell}_t(\bu)\\ 
&\leq \langle \bg_t,\bw_t - \bw_{t+1} \rangle + \langle \bg_t^+ , \bw_{t+1}-\bu \rangle~, 
\end{aligned}
\end{equation}
where $\bg_t \in \partial \ell_t(\bw_t)$, $\bg_t \in \partial \hat{\ell}_t(\bw_t)$, $ \bg_t^+ \in \partial \hat{\ell}_t(\bw_{t+1})$, the first inequality is true by the property of the truncated linear model, and the second inequality is from the convexity.
This decomposition can \emph{take into account part of the geometry of the function through $\bg_t^+$}, which quantifies how far we are from the infimum of $\ell_t$. Note that the decomposition itself is very general and does \emph{not} require $\hat \ell$ to be a truncated linear loss: this structure is primarily used to form more efficient algorithms.

\noindent\textbf{Implicit Coin-Betting}
To leverage this decomposition, we now define a modified notion of the \emph{wealth} quantity described in Section~\ref{sec:def}.
Our key idea is to realize that \emph{the regret/reward duality is more general than previously thought}. In particular, we define
$\Wealth_0=\epsilon$ and $\Wealth_T \triangleq \Wealth_{T-1} - \langle \bg_t,\bw_t - \bw_{t+1} \rangle - \langle \bg_t^+,\bw_{t+1}\rangle$, to have
\begin{equation}
\label{eq:wealth}
\Wealth_T 
= \epsilon - \sum_{t=1}^T (\langle \bg_t,\bw_t - \bw_{t+1} \rangle + \langle \bg_t^+,\bw_{t+1}\rangle)~.
\end{equation}
This implies $\Regret_T(\bu) \leq \epsilon + \left\langle -\sum_{t=1}^T \bg_t^+,\bu \right\rangle - \Wealth_T$.
Suppose that we obtain a bound $\Wealth_T\geq \psi_T \left(-\sum_{t=1}^T \bg_t^+\right)$ for some $\psi_T$. Then, we can still use the Fenchel conjugate:
\begin{align*}
\Regret_T(\bu) - \epsilon 
\leq -\left\langle \sum_{t=1}^T \bg_t^+,\bu \right\rangle - \psi_T\left(-\sum_{t=1}^T \bg_t^+ \right)\leq \sup_{\by} \ \langle \by,\bu \rangle - \psi_T(\by)
=\psi_T^\star(\bu)~.
\end{align*}
Hence, it suffices to design an algorithm that guarantees a lower bound on $\Wealth_T$ to achieve a regret upper bound, \emph{even for our modified notion of wealth}. Moreover, given that our regret decomposition takes into account the geometry of the truncated linear losses, we can expect a regret guarantee that becomes tighter when we are close to the infimum of the functions $\ell_t$.

\begin{algorithm}[t]
   \caption{Parameter-free OCO with Truncated Linear Models}
   \label{alg:imp_coin}
\begin{algorithmic}[1]
   \STATE Initialize $\bbeta_1 \leftarrow \boldsymbol{0},\Wealth_0\leftarrow1,\eta_1 \leftarrow 1/3$
   \FOR{$t=1$ {\bfseries to} $T$}
   	\STATE Predict $\bw_{t} \leftarrow \bbeta_t \Wealth_{t-1}$
   	\STATE Receive $\ell_t(\bw_t)$ and $\bg_t \in \partial \ell_t(\bw_t)$
   	\STATE Calculate $\bg_t^+$ (see Section~\ref{sec:comp_h}). $\bg_t^+=h_t\bg_t$ by property (ii) of truncated linear models 
   	\STATE $\hat{\bbeta}_{t+1} \leftarrow \bbeta_t - \frac{\bg_t^+ + 2\bbeta_t(\|\bg_t\|^2-\|\bg_t^+-\bg_t\|^2)}{1/\eta_1+2\sum_{i=1}^{t-1} (\|\bg_{i}\|^2 - \|\bg_{i}^+ - \bg_{i}\|^2)}$
   	\STATE $ \bbeta_{t+1} \leftarrow \hat{\bbeta}_{t+1}/\max\left(1, 2\|\hat{\bbeta}_{t+1}\|\right)$
   	\STATE $\Wealth_t \leftarrow \Wealth_{t-1}\frac{1-\langle \bg_t, \bbeta_t\rangle}{1+(h_t-1)\langle \bg_t,\bbeta_{t+1}\rangle}$
   \ENDFOR
\end{algorithmic}
\end{algorithm}

We designed Algorithm~\ref{alg:imp_coin} to maximize the wealth in \eqref{eq:wealth}, yielding a regret bound in Theorem~\ref{thm:main}.

\begin{theorem}
\label{thm:main}
Assume $\ell_t(\bx),  t=1,\dots,T$, to be convex functions. Set $\Wealth_0 = \epsilon = 1$ and assume that $\|\bg_t\|\leq 1$ and $\bg_t^+ = h_t \bg_t$ where $h_t \in [0,1]$. 

Then, Algorithm~\ref{alg:imp_coin} satisfies

\begin{align*}
\Regret_T(\bu) = &\mathcal{O}\left(  \max\left\{ \|\bu\| \ln\left(\|\bu\|\left(1+\sum_{t=1}^T \|\bg_t\|\|\bg_t^+\| \right)\right)\right.\right. ,\\
& \left.\left.\|\bu\|\sqrt{\sum_{t=1}^T (2\|\bg_t\|\|\bg_t^+\|-\|\bg_t^+\|^2) \cdot \ln\left(1+\|\bu\|\sum_{t=1}^T (2\|\bg_t\|\|\bg_t^+\|-\|\bg_t^+\|^2)\right)} \right\}  \right).
\end{align*}
\end{theorem}

To convey the main ideas, here we present a proof sketch, the full proof is included in the Appendix.

\begin{proof}[Proof sketch]
We first lower bound the wealth of the algorithm. From the definition of the wealth \eqref{eq:wealth} and the fact that the algorithm predicts with $\bw_t=\bbeta_t\Wealth_{t-1}$, we have
\begin{equation}
\label{eq:wealth_def}
\Wealth_t = \Wealth_{t-1} - \langle \bg_t, \bw_t-\bw_{t+1}\rangle - \langle \bg_t^+, \bw_{t+1}\rangle 
\Rightarrow \Wealth_t = \frac{\Wealth_{t-1}(1-\langle \bg_t,\bbeta_{t}\rangle)}{1+\langle \bg_t^+ - \bg_t,\bbeta_{t+1}\rangle}~.
\end{equation}
This implies that $\ln \Wealth_T =\ln \epsilon + \sum_{t=1}^T (\ln(1-\langle \bg_t, \bbeta_{t}\rangle) - \ln(1+\langle \bg_t^+-\bg_t,\bbeta_{t+1}\rangle))$. It is possible to show that $\ln \Wealth_T -\ln \epsilon$ can be lower bounded as
\begin{equation}
\label{eq:log_wealth}
\begin{aligned}
\sum_{t=1}^T &\left(\ln(1-\langle \bg_t,\bbeta_{t}\rangle) - \ln(1+\langle \bg_t^+-\bg_t,\bbeta_{t+1}\rangle) \right) \\
&\geq \sum_{t=1}^T \left[-\langle \bg^+_t, \bbeta_t\rangle -  (\|\bg_t\|^2 - \|\bg_t^+ - \bg_t\|^2 )\|\bbeta_t\|^2 -2\|\bg_t\| \|\bbeta_{t+1}-\bbeta_t\|\right]~. 
\end{aligned}
\end{equation}
So, $\bbeta_t$ is designed to be the output of running OGD (Online Gradient Descent) on $\mu_t$ strongly-convex losses $f_t(\bbeta) \triangleq \langle \bg_t^+,\bbeta\rangle + \frac{\mu_t}{2}\| \bbeta_t \|^2$, where $\mu_t=2(\|\bg_t\|^2 - \|\bg_t-\bg_t^+\|^2)$ with $\bbeta\in B$, $B=\{ \bx | \|\bx\|\leq 1/2\}$, and stepsizes $\eta_t = \frac{1}{1/\eta_1 + \sum_{i=1}^{t-1}\mu_i}$. Standard OGD analysis provides the lower bound for $\sum_{t=1}^T \langle \bg^+_t, \bbeta_t\rangle + (\|\bg_t\|^2 - \|\bg_t^+ - \bg_t\|^2 )\|\bbeta_t\|^2$. Besides, $\| \bbeta_{t+1} -\bbeta_t \|$ is upper bounded by $\frac{3\|\bg_t^+\|}{1+2\sum_{i=1}^t \|\bg_i\|\|\bg_i^+\|}$. Combining all pieces together leads to the lower bound 
\[
\ln\Wealth_T \geq -3/2 -7.25\ln \left(1+2\sum_{t=1}^T \|\bg_t\|\| \bg_t^+\| \right) + \min\left\{\frac{\left\|\sum_{t=1}^{T}\bg_t^+\right\|}{4}, \frac{\left\|\sum_{t=1}^T \bg_t^+\right\|^2}{2\sum_{t=1}^T \mu_t} \right\}~.
\]

A lower bound on $\Wealth_T$ indicates an upper bound on regret. Now, we derive the upper bound on the regret from the Fenchel conjugate of the function above.
\end{proof}

Note that the bounded subgradient assumption is a known requirement shared by all parameter-free algorithms, see lower bound in \citet{CutkoskyB17}. However, the limitation is milder than it seems at first blush: this Lipschitz bound can actually be over-estimated by a factor of $\sqrt{T}$ before significant damage is done to the regret bound. This can be seen by observing that other than an $O(\log{T})$ term, our regret bounds scale with the observed norms of the gradients. Thus, the limitation is actually rather benign - we simply assume a bound of 1 to simplify equations.

\noindent\textbf{Comparison with Parameter-Free Bounds}
Previous work of \citet{CutkoskyO18} achieved a regret bound of $\mathcal{O}\left(\|\bu\|\sqrt{ \sum_{t=1}^T \|\bg_t\|^2}\right)$ which has the optimal worst-case dependence on $\|\bg_t\|$~\citep{AbernethyBRT08,Cutkosky18}. In Theorem~\ref{thm:main}, we obtain a regret bound depending on $\|\bg_t\|^2-\|\bg_t^+ - \bg_t\|^2=\|\bg_t\|^2(2h_t-h_t^2)$. So, as long as the algorithm goes close to the hinge corner of the truncated linear model, it will yield an $h_t < 1$ and a smaller regret. Intuitively, this should be expected to occur whenever it is possible to obtain small loss as obtaining small loss requires reaching the hinge of the truncated linear model.

\noindent\textbf{Relation to Implicit Updates}
Truncated linear models were introduced as an approximation of the implicit updates~\citep{AsiD19}.
In this view, it is instructive to compare the dependency on the subgradients in Theorem~\ref{thm:main} and the regret bounds for implicit updates. For example, \citet[Theorem 2]{McMahan10} gives a regret guarantee for FTRL with implicit updates and non-adaptive regularizer that depends on $\langle \bg_t-\frac{1}{2}\bg^+_t,\bg_t^+\rangle$. This quantity is exactly $\frac{1}{2}(\|\bg_t\|^2 - \|\bg_t^+-\bg_t\|^2)$ that appears in Theorem~\ref{thm:main}. This supports the idea that the decomposition in \eqref{eq:decomposition} ``emulates'' the idea of implicit updates in parameter-free algorithms. However, there is a subtle difference: in standard implicit updates $\bg_t^+ \in \partial \ell_t(\bw_{t+1})$ is a subgradient of the original loss function. Instead, here $\bg_t^+ \in \partial \hat{\ell}_t(\bw_{t+1})$, so it is a subgradient of the truncated linear model. We can see this as a price we pay to obtain a smaller computational complexity compared to standard implicit updates. 

\noindent \textbf{Comparison with OMD with truncated linear model} To the best of our knowledge, there are actually no regret guarantees with OMD with truncated linear models in the literature (\citet{AsiD19} do not consider the adversarial setting). However, it is quite likely that OMD with truncated models can achieve an implicit regret similar to that reported by \citet{McMahan10} \emph{subject to oracle tuning of the learning rates}. Our results match this benchmark in the dependency on $\bg_t$ and $\bg_t^+$ and improve in the dependency on $\|\bu\|$ since we do not require oracle tuning of the learning rate. 

\noindent\textbf{No Overshooting Property} We now prove that the proposed algorithm never overshoots the minimum of the truncated linear loss. Moreover, in the case that the minimum of $\hat{\ell}_t$ coincides with the minimum of $\ell_t$, we end up exactly in the minimum, as illustrated by Figure~\ref{fig:illustration}.

\begin{theorem}\label{thm:noovershoot}
Under the assumptions of Theorem~\ref{thm:main} and the notation of Algorithm~\ref{alg:imp_coin}, assume that $\bg_t\neq \boldsymbol{0}$. Then, the update $\bw_{t+1}$ can never land on the flat part of the loss $\hat{\ell}_t$, but only on its linear part or in the corner. 
\end{theorem}

\begin{proof}
The statement is equivalent to showing that $\ell_t(\bw_t) + \langle \bg_t, \bw_{t+1} -\bw_t\rangle \geq 0$ from the definition of $\hat{\ell}_t$. We prove it by contradiction. Let's assume that $\ell_t(\bw_t) + \langle \bg_t, \bw_{t+1} -\bw_t\rangle<0$. Then, we would have $\hat{\ell}_t(\bw_{t+1})=0$ and $\bg_t^+=\boldsymbol{0}$ (equivalently $h_t=0$). In turn, this would imply $\bbeta_{t+1}=\bbeta_t$ and $\Wealth_t=\Wealth_{t-1}$. So, we would have $\bw_{t+1}=\bw_t$ which is impossible because $\bg_t\neq \boldsymbol{0}$.
\end{proof}

Note that we assume $\bg_t\neq \boldsymbol{0}$ in Theorem~\ref{thm:noovershoot} since when $\bg_t = \boldsymbol{0}$ the algorithm is already in the corner.

\subsection{Computation of $h_t$}
\label{sec:comp_h}

The next challenge is how to find $h_t$. This is the only part of the Algorithm that uses the truncated linear model structure: the analysis Theorem~\ref{thm:main} actually applies to \emph{any} losses for which $\bg_t^+=h_t\bg_t$. Truncated linear models combine this favorable property with the additional property that it is possible to efficiently compute $h_t$. The argument of Theorem \ref{thm:noovershoot} shows that we cannot be in the flat region. By inspection of the updates, there are two achievable cases: in the first case $h_t=1$ and we are not in the corner of the truncated model, while in the second case, we are in the corner. Hence, as a first step, we posit that $h_t=1$, calculate $\bw_{t+1}$ and see if indeed $\bg_t^+=\bg_t$. If this is not the case, then the solution must be in the corner and $h_t \in[0,1)$. By definition $\bw_{t+1}=\bbeta_{t+1}\Wealth_t$, where

\[
\bbeta_{t+1}= \prod_B (\bbeta_t - \eta_t (h_t\bg_t+2\bbeta_t \|\bg_t\|^2 (2h_t-h^2_t) )
\text{ and }
\Wealth_t = \Wealth_{t-1} \frac{1-\bg_t\bbeta_t}{1+(\bg_t^+ -\bg_t)\bbeta_{t+1}}~.
\]

Thus, $\bw_{t+1}$ is a function of $h_t$. Assuming $\inf_{\bw} \ell_t(\bw)=0$ w.l.o.g. for loss functions bounded from below, we are looking for $h_t$ that makes 
\begin{equation}
\label{eq:h_eq}
\ell_t(\bw_t)+\langle \bg_t, \bw_{t+1} - \bw_t\rangle = 0~.
\end{equation}
Although we could solve for $h_t$ via bisection, there is no closed form solution due to the projection of $\bbeta_{t+1}$. Thus, we next propose a more complex algorithm with a closed form equation for $h_t$.

%% file: closed_form_sol.tex
\section{Variant with Closed-form Update}
\label{sec:closed_form_sol}

In this section, we introduce the Algorithm~\ref{alg:imp_coin_closed_form}: a variant of Algorithm~\ref{alg:imp_coin} that has a closed form update. The key steps are still the same, but here we want to remove the projection step on $\bbeta_t$.
In this way, the expression of $\bw_{t+1}$  depends on a simple polynomial in $h_t$. In turn, to remove the projection step, we change the update of $\bbeta_t$ so that its norm is always assured to be bounded. 
\begin{algorithm}[t]
\caption{Parameter-free OCO with Truncated Linear Models -- closed form update}
\label{alg:imp_coin_closed_form}
\begin{algorithmic}[1]
%   \STATE {\bfseries Input:} 
\STATE Initialize $\bbeta_1 \leftarrow \boldsymbol{0},\eta_1 \leftarrow \frac{1}{2C}, C \leftarrow 9, \Wealth_0 \leftarrow 1$
\FOR{$t=1$ {\bfseries to} $T$}
    \STATE Predict $\bw_{t} \leftarrow \bbeta_t \Wealth_{t-1}$
    \STATE Receive $\ell_t(\bw_t)$ and $\bg_t \in \partial \ell_t(\bw_t)$
    \STATE Calculate $h_t$ (see Section~\ref{sec:closed_form_update})
    \IF{$\|\bbeta_t\| < \frac{3}{8}$}
    \STATE $\bbeta_{t+1} \leftarrow \bbeta_t - \eta_t(\bg_t^+ + 2\bbeta_t(2\|\bg_t\| \|\bg_t^+\| - \|\bg_t^+\|^2))$
    \STATE $\frac{1}{\eta_{t+1}} \leftarrow \frac{1}{\eta_t}+2(\|\bg_t\|^2 - \| \bg_t^+ - \bg_t\|^2)$
    \ELSE
    \STATE $\bbeta_{t+1} \leftarrow \bbeta_t - \eta_t 2C \| \bg_t^+ \| \bbeta_t$
    \STATE $\frac{1}{\eta_{t+1}} \leftarrow \frac{1}{\eta_t}+ 2C\|\bg_t^+\|$
    \ENDIF
    \STATE $\Wealth_t \leftarrow \Wealth_{t-1}\frac{1-\langle \bg_t,\bbeta_t\rangle}{1+(h_t-1)\langle \bg_t,\bbeta_{t+1}\rangle}$
\ENDFOR
\end{algorithmic}
\end{algorithm}
For Algorithm~\ref{alg:imp_coin_closed_form}, we can prove the following guarantee. We present a proof sketch, while the full proof is in the Appendix.
\begin{theorem}
\label{thm:main_closed_form}
Assume $\ell_t(x),  t=1,\dots,T$, to be convex functions. Set $\Wealth_0 = \epsilon = 1$, $C=9$ and $1/\eta_1=2C$, and assume that $\|\bg_t\|\leq 1$ and $\bg_t^+ = h_t \bg_t$ where $h_t \in [0,1]$. 

Then, for all $\bu\in \R^d$, Algorithm~\ref{alg:imp_coin_closed_form} satisfies the same bound as in Theorem~\ref{thm:main}, up to constants hidden in the big O notation.
\end{theorem}

\begin{proof}[Proof sketch]

As we stated above,  due to the projection step in line 6 of Algorithm~\ref{alg:imp_coin}, $h_t$ can not be solved in a closed form.  To overcome this, we design a new update rule that guarantees that $\bbeta_t$ will always end up in the ball $B = \{\bx : \|\bx\|\leq \frac{1}{2}\}$ to avoid the projection step. 

In Algorithm 1, $\beta_t$ is the output of running OGD on strongly-convex losses $f_t(\beta)$. However, when $\beta_t$ is close to 1/2, the next iteration, $\beta_{t+1}$, could go too far so that a projection step can be necessary. To avoid this, we intricately design an update rule that when $\| \beta_t \| \geq 3/8$ indicating that $\beta_t$ is close to the boundary of the ball, the next iteration will shrink it a little bit, to make sure that it stays in the ball $B$. The new update rule is the output of running OGD on strongly-convex losses $\phi_t(\beta)$. 
In the following, we introduce the sketch of the proof. 

$\bbeta_t$ is the output of OGD with $\eta_t = \frac{1}{1/\eta_1+\sum_{i=1}^{t-1}\mu_i}$ on the strongly-convex losses $\phi_t(\bbeta)$:
\[
\phi_t(\bbeta) =
\begin{cases}
f_t(\bbeta), & \text{ if } \| \bbeta_t\| < \frac{3}{8},\\
C\|\bg_t^+\| \|\bbeta\|^2, & \text{ if }  \frac{3}{8} \leq \|\bbeta_t\| \leq \frac{1}{2},
\end{cases}
\]
where $\bbeta_t\in B$ and $B = \{\bx : \|\bx\|\leq \frac{1}{2}\}$. $\phi_t(\bbeta)$ is $\mu_t$ strongly convex. $\bbeta_1^\star \triangleq \arg\min_{\bbeta\in B_1}\sum_{t=1}^T f_t(\bbeta)$, where $B_1 = \{\bx : \|\bx\|\leq \frac{1}{4}\}$. The intricate design of $\phi_t(\bbeta)$ allows to say that if $\bbeta_t\geq 3/8$, $\|\bbeta_{t+1}\| = (1-\eta_t 2C \|\bg_t^+\|)\|\bbeta_t\|$ will shrink if $\eta_t\leq 1/{2C}$; if $\bbeta_t\leq 3/8$, $\|\bbeta_{t+1}\|$ will stay in $B$ if $\eta_t$ is small enough. Therefore, Algorithm~\ref{alg:imp_coin_closed_form} guarantees $\|\bbeta_t\| \leq \frac{1}{2}$ for all $t = 1, \dots,T$, which removes the projection step on $\bbeta_t$, and gives rise to the closed form updates. 

Furthermore, for $C \geq 9$, the regret of $\phi_t(\bbeta_t)$ upper bounds the regret of $f_t(\bbeta_t)$: $f_t(\bbeta_t) - f_t(\bbeta_1^\star) \leq \phi_t(\bbeta_t) - \phi_t(\bbeta_1^\star)$.
We upper bound $\sum_{t=1}^T \phi_t(\bbeta_t)- \phi_t(\bbeta_1^\star)$ by standard OGD analysis, which implies a bound on $\sum_{t=1}^T f_t(\bbeta_t)$.
Combining the upper bound of $\|\bbeta_{t+1} - \bbeta_t \|$ and $\sum_{t=1}^T f_t(\bbeta_t)$ lower bounds wealth by \eqref{eq:log_wealth}.
The rest of the proof is similar to the proof of Theorem~\ref{thm:main}.
\end{proof}

We also note that the non-overshooting property holds for this algorithm too. The proof is exactly the same as before and it is omitted.

\begin{theorem}
\label{thm:noovershoot2}
Under the assumptions of Theorem~\ref{thm:main_closed_form} and the notation of Algorithm~\ref{alg:imp_coin_closed_form}, assume that $\bg_t\neq \boldsymbol{0}$. Then, the update $\bw_{t+1}$ can never land on the flat part of the loss $\hat{\ell}_t$, but only on its linear part or on the corner. 
\end{theorem}

\subsection{Computation of $h_t$ with Closed Form Solution}
\label{sec:closed_form_update}

Now, we show how to obtain a closed form expression for the update in Algorithm~\ref{alg:imp_coin_closed_form}. As before, first we tentatively set $h_t=1$ and check if $\bg^+_t=\bg_t$. If yes, then $h_t=1$ and we can compute $\bw_{t+1}$. If not, thanks to Theorem~\ref{thm:noovershoot2}, we know that $\bw_{t+1}$ lands in the corner of $\hat{\ell}_t$ and we need to compute $h_t \in [0,1)$. In this case we are looking for the $h_t$ such that $\bw_{t+1}$ satisfies \eqref{eq:h_eq}.

Let $A = \langle \bg_t,\bw_t\rangle - \ell_t(\bw_t)$, $B=\Wealth_t(1-\langle \bg_t,\bbeta_t\rangle)$. We consider two cases based on $\bbeta_t$.

If $\|\bbeta_t\|\leq \frac{3}{8}$, we have that $\langle \bg_t,\bbeta_{t+1}\rangle = \langle \bg_t, \bbeta_t-\eta_t(\bg_t^+ + 2\bbeta_t(2h_t\|\bg_t\|^2- h_t^2\|\bg_t\|^2))\rangle$. Let $D = 2\eta_t\|\bg_t\|^2\langle \bg_t,\bbeta_t\rangle$, so \eqref{eq:h_eq} becomes a cubic equation of $h_t$ that has closed form solution:
\begin{align*}
&-AD h_t^3 +(2AD+A\eta_t\|\bg_t\|^2+(A+B)D)h_t^2+(-(A+B)\eta_t\|\bg_t\|^2-2(A+B)D\\
&\quad -A\langle \bg_t,\bbeta_t\rangle)h_t+ (A+B)\langle \bg_t,\bbeta_t\rangle - A=0~.
\end{align*}

If $ \frac{3}{8} \leq\|\bbeta_t\|\leq \frac{1}{2}$, we have $\langle \bg_t,\bbeta_{t+1}\rangle = \langle \bg_t,\bbeta_t\rangle(1-2C\eta_t\|\bg_t\|h_t)$. Let $D = 2C\eta_t\|\bg_t\| \langle \bg_t,\bbeta_t\rangle$. So, \eqref{eq:h_eq} can be rewritten as the following quadratic equation of $h_t$ and again it has closed form solution:
$
AD h_t^2+(-A\langle \bg_t,\bbeta_t\rangle - (A+B)D)h_t + (A+B)\langle \bg_t,\bbeta_t\rangle-A =0
$.

%% file: coor.tex
\section{Tighter Regret Guarantee through Coordinate-wise Updates}
\label{sec:coor}

In this section, we introduce a coordinate-wise variant of Parameter-free OCO with truncated linear models. This is a simple extension of Algorithm~\ref{alg:imp_coin_closed_form} by considering each coordinate as a different 1-d OCO algorithm. The advantage is that this regret bound is even \emph{tighter} than the bound of Theorem~\ref{thm:main}. Here we present the Theorem~\ref{thm:coor}. (The proof can be found in the Appendix.)
We use $\beta_{t,i},~u_i,~g_{t,i}\in \mathbb{R}$ to represent the i-th element of the vector correspondingly.  

\begin{algorithm}[t]
\caption{Parameter-free OCO with Truncated Linear Models -- coordinate-wise update}
\label{alg:imp_coin_coor}
\begin{algorithmic}[1]

\STATE Initialize $\bbeta_1 \leftarrow \boldsymbol{0},\eta_1 \leftarrow \frac{1}{2C}\cdot \boldsymbol{1}, C \leftarrow 9, \Wealth_0 \leftarrow \boldsymbol{1}\in \mathbb{R}^d$
\FOR{$t=1$ {\bfseries to} $T$}
    \STATE Predict $\bw_{t} \leftarrow \bbeta_{t} \odot \Wealth_{t-1}$
    \STATE Receive $\ell_t(\bw_t)$ and $\bg_t \in \partial \ell_t(\bw_t)$
    \STATE Calculate $h_t$ s.t. $
\ell_t(\bw_t)+\langle \bg_t, \bw_{t+1} - \bw_t\rangle = 0$ (see Section~\ref{sec:comp_h})
    \FOR{$i=1$ {\bfseries to} $d$}
    \IF{$|\beta_{t,i}| < \frac{3}{8}$}
    \STATE $\beta_{t+1,i} \leftarrow \beta_{t,i} - \eta_{t,i}(g_{t,i}^+ + 2\beta_{t,i}(2g_{t,i}g_{t,i}^+ - (g_{t,i}^+)^2))$
    \STATE $\frac{1}{\eta_{t+1,i}} \leftarrow \frac{1}{\eta_{t,i}}+2(g_{t,i}^2 - (g_{t,i}^+ - g_{t,i})^2)$
    \ELSE
    \STATE $\beta_{t+1,i} \leftarrow \beta_{t,i} - \eta_{t,i} 2C | g_{t,i}^+ | \beta_{t,i},$
    \STATE $ \frac{1}{\eta_{t+1,i}} \leftarrow \frac{1}{\eta_{t,i}}+ 2C|g_{t,i}^+|$
    \ENDIF
    \STATE $\Wealth_{t,i} \leftarrow \Wealth_{t-1,i}\frac{1-g_{t,i}\beta_{t,i}}{1+(h_t-1) g_{t,i}\beta_{t+1,i}}$
	\ENDFOR
\ENDFOR
\end{algorithmic}
\end{algorithm}

\begin{theorem}
\label{thm:coor}
Assume $\ell_t(x),  t=1,\dots,T$, to be convex functions. Set $\Wealth_0 = \epsilon = \boldsymbol{1}$, $C=9$. For $i=1,\dots,d$, $1/\eta_{1,i}=2C$ and assume that $|g_{t,i}|\leq 1$, $\bg_t^+ = h_t \bg_t$ where $h_t \in [0,1]$. 
Then, for all $\bu\in \R^d$, Algorithm~\ref{alg:imp_coin_coor} guarantees
\begin{align*}
\Regret_T(\bu)&= \sum_{i=1}^{d} \epsilon + \mathcal{O}\left(  \max\left\{ |u_i| \ln\left(|u_i|\left(1+\sum_{t=1}^T |g_{t,i}| |g_{t,i}^+| \right)\right)\right.\right. ,\\
& \quad \left.\left.|u_i|\sqrt{\sum_{t=1}^T (2|g_{t,i}| |g_{t,i}^+|-|g_{t,i}^+|^2) \cdot \ln\left(1+|u_i|\sum_{t=1}^T (2|g_{t,i}||g_{t,i}^+|-|g_{t,i}^+|^2)\right)} \right\}  \right)~.
\end{align*}
\end{theorem}
To see how this bound is more desirable than Theorem~\ref{thm:main}, first notice that the bound obtains some adaptivity to  $L_1$ geometry: if $G_\infty=\max\|g_t\|_{\infty}$, then the bound is at most $\tilde{\mathcal{O}}(\|u\|_1G_\infty{\sqrt{T}})$. 
Furthermore, by application of Cauchy-Schwarz (twice!) we can see that the bound is at most an additive $\epsilon d$ larger than Theorem~\ref{thm:main} - and could be much smaller if either application of Cauchy-Schwarz is loose. Thus, this bound is never much worse than that of Theorem~\ref{thm:main}, but has the further desirable property that one can add a large number of ``irrelevant'' dimensions for which $u_i=0$ without harming the bound.

Unfortunately, we no longer have a closed form expression for $h_t$ due to the coupling of the coordinates. However, at each iteration, we can find a $\delta$-approximation to $h_t$ using $\mathcal{O}(\log(1/\delta))$ steps of bisection and a single gradient oracle call.  By ``bisection'' we mean a binary-search style algorithm: given a guess for $h_t$, we can compute if the true value is lower or higher than the guess by computing what the update would be if the guess were correct and checking if we have overshot the corner of the truncated bound. Though this inflates the cost of an update by $\mathcal{O}(\log(1/\delta))$, this is still \emph{significantly} more efficient than the $\text{poly}(T)$ oracle calls required to run FTRL with truncated linear models directly at each step. This small extra computation cost is a price we pay for better theoretical as well as empirical results.

%% file: exp.tex
\section{Empirical Evaluation}
\label{sec:exp}

While our main contribution is theoretical, here we evaluate the empirical performance of Algorithm~\ref{alg:imp_coin_closed_form} and Algorithm~\ref{alg:imp_coin_coor} to show their practical potential.
We will denote the algorithms as Implicit Coin, and Coordinate-wise Implicit Coin. We would also like to stress that we used Algorithm~\ref{alg:imp_coin_closed_form} and Algorithm~\ref{alg:imp_coin_coor} as they are, with the choice of the hyperparameters directly given by theory, i.e., $\eta_0$, $C$, $\Wealth_0$. It is quite possible that these choices were not optimal. We do this on purpose: we want to demonstrate how robust parameter-free algorithms are, even with theory-derived constants.

We compare SGD, SGD with truncated models (aProx)~\citep{AsiD19}, SGD with Importance Weight Aware updates (IWA)~\citep{KarampatziakisL11}, Coin-betting algorithm (Coin)~\citep{OrabonaP16}, Coin-betting with ODE updates (CODE)~\citep{ChenLO22}, COntinuous COin Betting (COCOB)~\citep{OrabonaT17}. 

We tested the algorithms on real-world datasets from the LIBSVM website~\citep{ChangL01} and OpenML~\citep{VanschorenVBT2013}. 2dPlane, CPU-act, and Houses are classification tasks, Rainfall, Bank32nh, and House-8L are regression tasks. (More information about datasets is in Appendix). We standardize and pre-process the samples, normalizing them to unit norm vectors. We shuffle the data and separate into a training set ($70\%$), validation set ($15\%$), and test set ($15\%$).

\begin{figure}[t]
\centering
\includegraphics[width=0.33\textwidth]{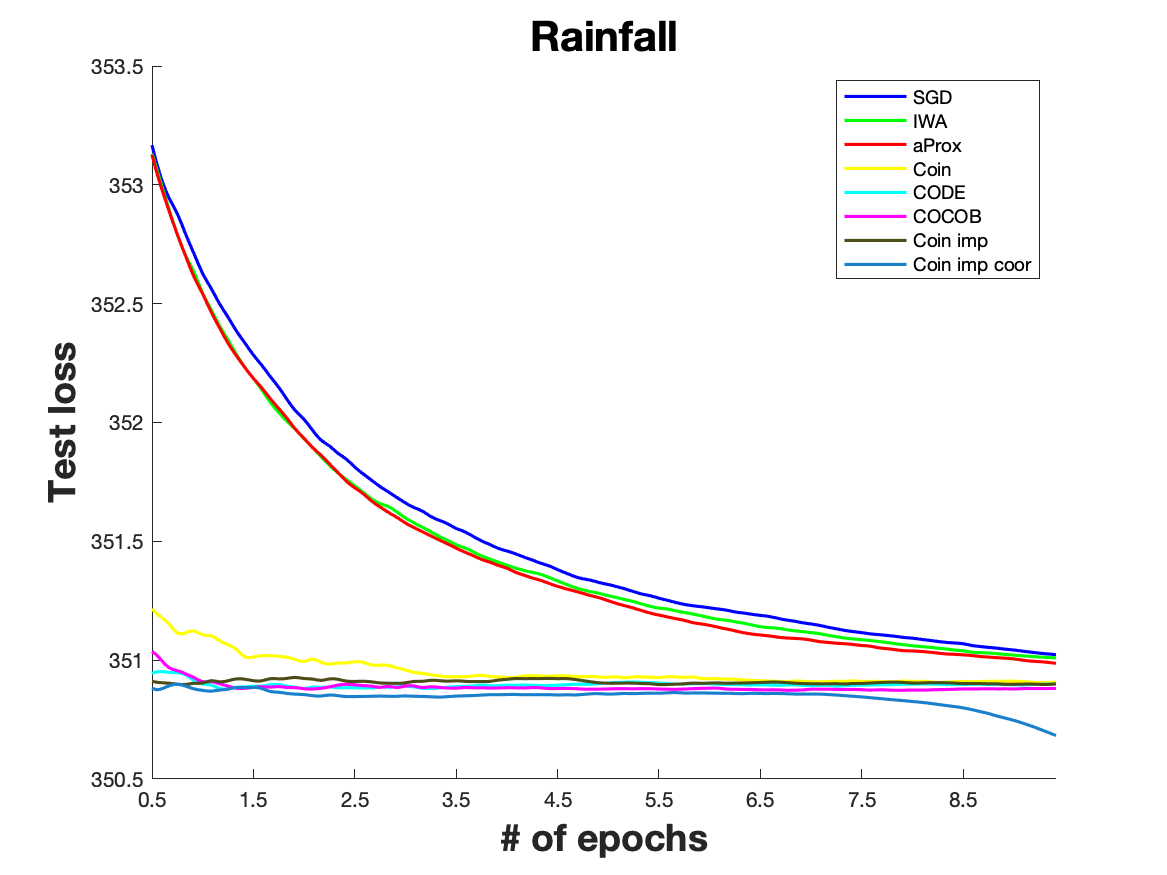} \hfill
\includegraphics[width=0.33\textwidth]{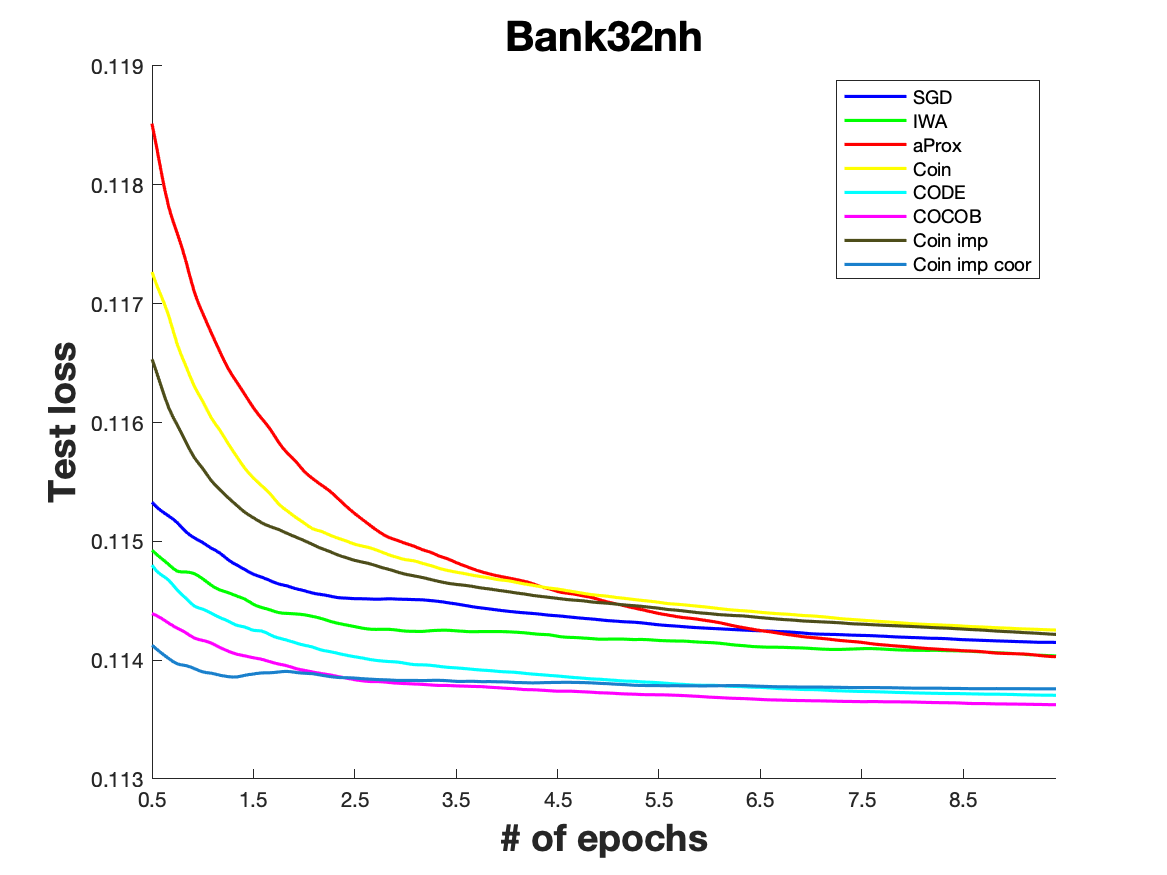}\hfill
\includegraphics[width=0.33\textwidth]{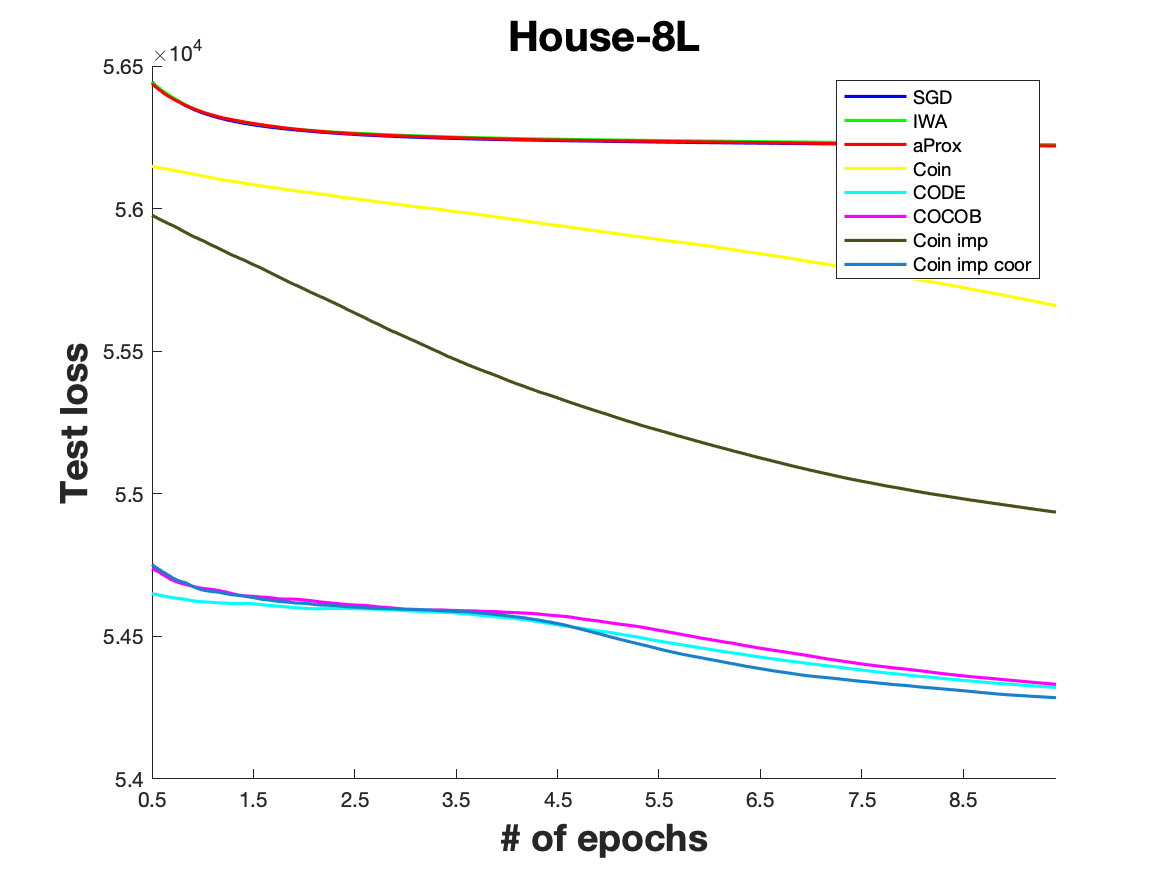}
\caption{Regression tasks: Mean test loss versus epochs.}
\label{fig:reg_plot}
\end{figure}

\begin{figure}[h]
\centering
\includegraphics[width=0.33\textwidth]{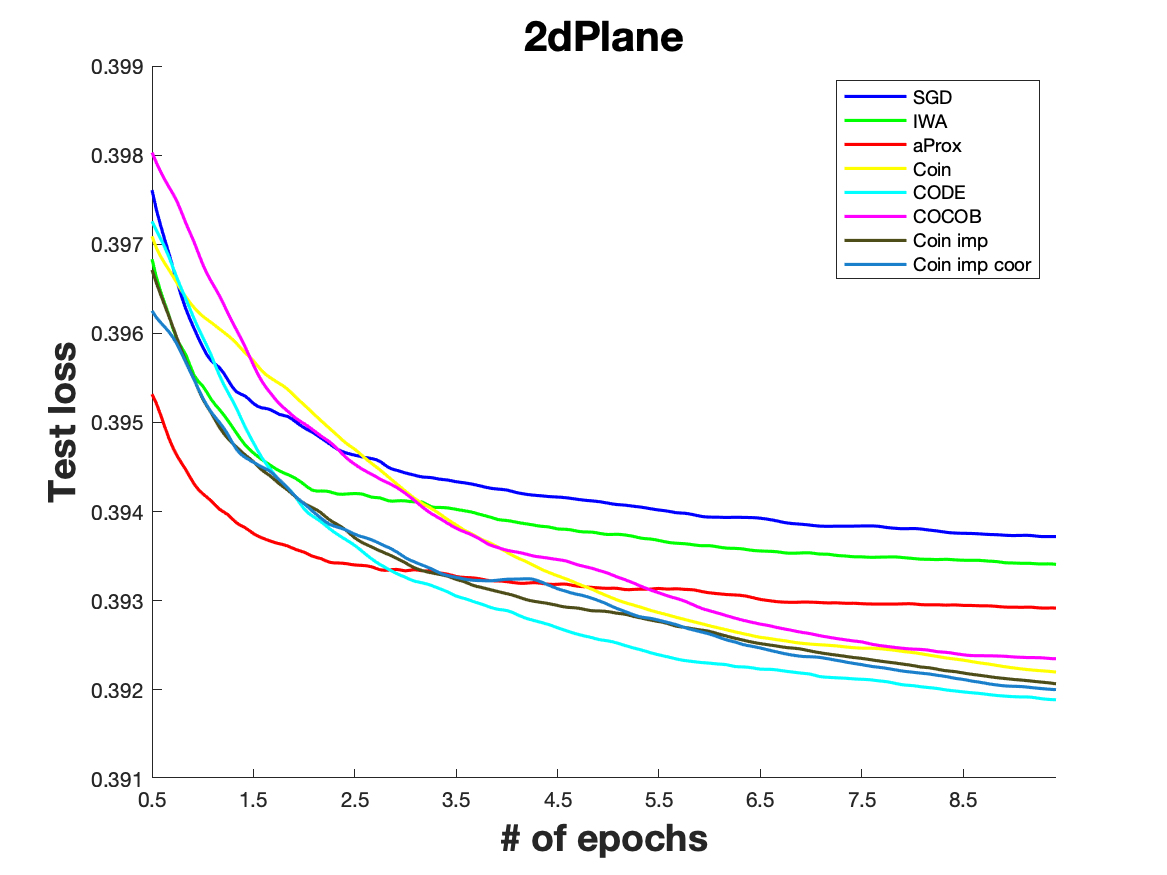} \hfill
\includegraphics[width=0.33\textwidth]{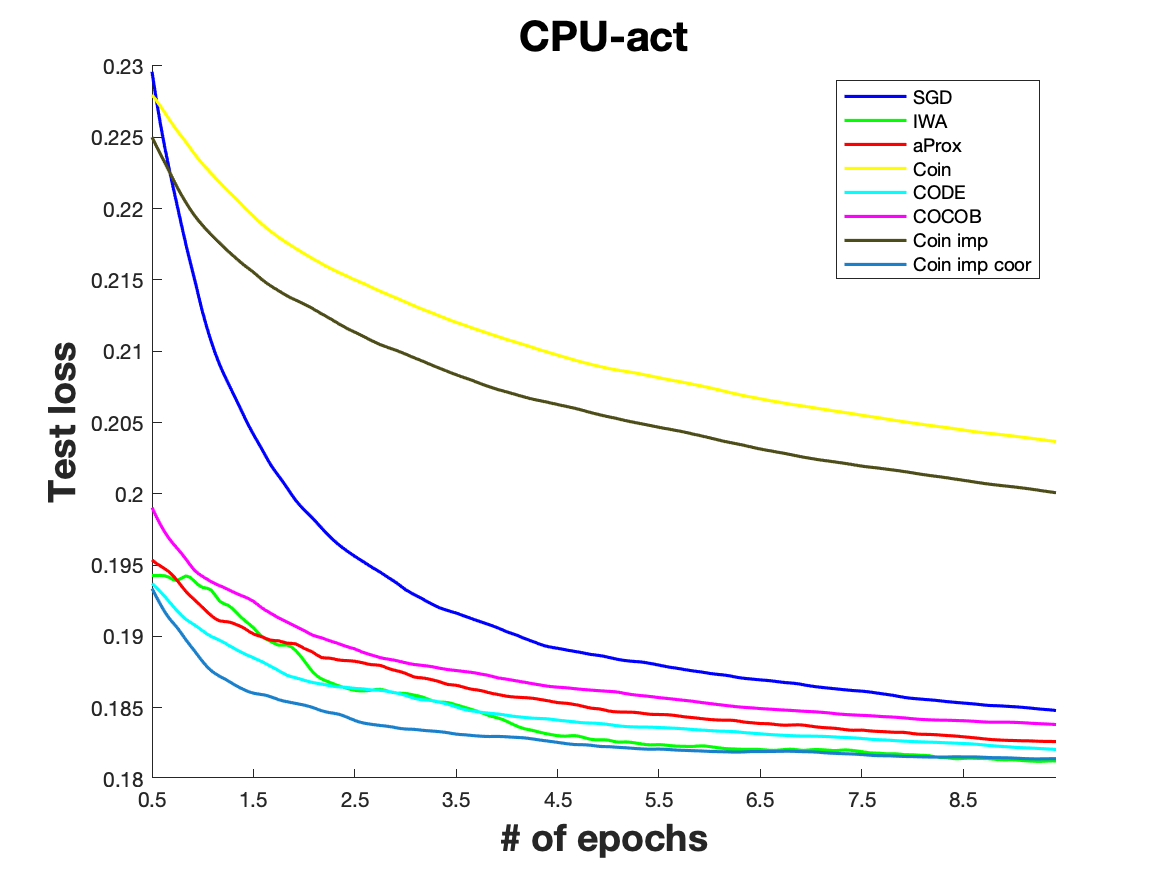}\hfill
\includegraphics[width=0.33\textwidth]{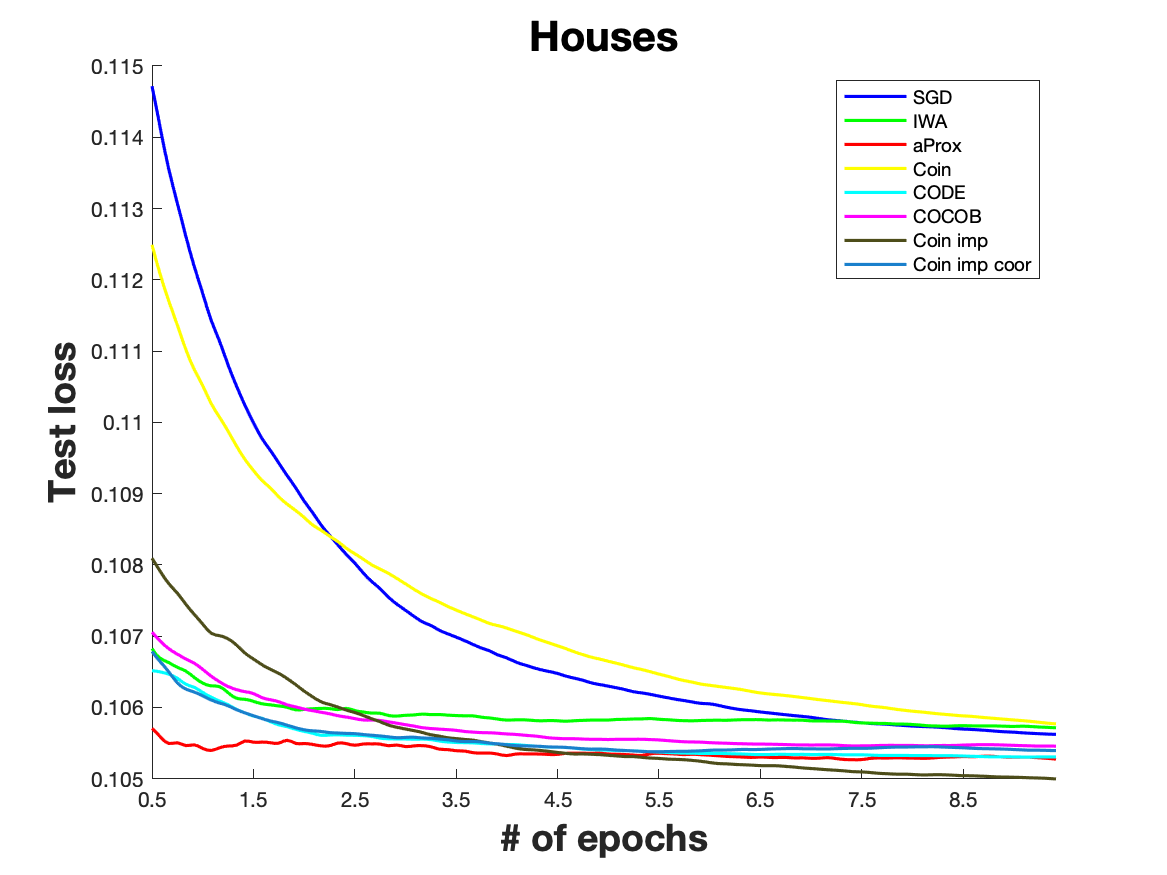}
\caption{Classification tasks: Mean test loss versus epochs.}
\label{fig:class_plot}
\end{figure}

For SGD, aProx and IWA, we tune the initial step size $\eta_0$ and consider stepsize sequence of the form: $\eta_k = \eta_0 / \sqrt{k}$. All the parameter-free algorithms do not have parameters to tune. Implicit Coin has a closed form solution for $h_t$, so the implementation is as efficient as SGD. For CODE and Coordinate-wise Implicit Coin, we used a bisection algorithm to find $h_t$.

We considered linear predictors trained with the hinge loss for classification, and with the absolute loss for regression.
We trained all algorithms for 10 epochs over the training data. Each epoch requires running through the total shuffled training set. All the experiments are repeated 3 times, we report the average of the 3 repetitions. For all the algorithms that require a learning rate, on every dataset and for each repetition, we use the validation loss to choose the best learning rate, train using that learning rate, test on the test set, and report the test loss averaged over 3 repetitions. 

Figure~\ref{fig:reg_plot} and Figure~\ref{fig:class_plot} show the average of test losses versus number of epochs. In all experiments, Coordinate-wise Implicit Coin has a performance that is superior or comparable to all the competitors. It is worth remembering that the algorithms with a learning rate were tuned on the validation set, while the parameter-free algorithms were not tuned in any way. Moreover, in all cases, Implicit Coin has a better performance than Coin. Given that their key difference is the truncated models in Implicit coin, this directly supports the advantage of these updates over linearized ones. 

More in detail, on regression tasks (Figures~\ref{fig:reg_plot}), the performance of Coordinate-wise Implicit Coin is superior to the other competitors at the end of the training on Rainfall and House-8L. COCOB, CODE, and Coordinate-wise Implicit Coin outperform the other competitors significantly on Houses-8L, and their performances are close to each other on Bank32nh. On the classification tasks (Figure~\ref{fig:class_plot}), Coordinate-wise Implicit Coin and Implicit Coin achieve essentially the optimal performance on CPU-act and Houses correspondingly. On 2dPlane, Coordinate-wise Implicit Coin, Implicit Coin, and CODE perform almost equally satisfying, and they outperform the other competitors. 

%% file: conc.tex
\section{Conclusion}
\label{sec:conc}

In this paper, we present new parameter-free algorithms utilizing a better convex lower bound: the truncated linear model. \emph{We overcome the theoretical difficulties of using truncated linear models in parameter-free algorithms with a new regret decomposition}. Our regret bounds are analogous to bounds achieved by implicit methods. Besides, we propose a variant of our algorithm that has a very efficient closed form update rule, matching the runtime of gradient descent. Finally, we provide a coordinate-wise variant with tighter regret bounds.

In the future, we would like to explore the additional possibilities offered by the new decomposition of regret. For example, we would like to overcome the limitation of the current per-coordinate formulation and explore the possibility to design a per-coordinate variant that uses truncated linear models, retaining the closed form update. Besides, considering the good empirical performance of CODE and the similarity in the spirit of CODE and Implicit Coin, we are interested in the possibility of studying the theoretical properties of CODE.

%% file: appendix.tex
\appendix
\label{sec:appendix}

\section{Proof of Theorem~\ref{thm:main}}
Before we prove Theorem~\ref{thm:main}, we first introduce some technical Lemmas that will be used in the proof. 

\begin{lemma}
\label{lemma:log}
Under the assumptions of Theorem~\ref{thm:main} and the notation of Algorithm~\ref{alg:imp_coin}, we have
$\|\bbeta_{t+1}-\bbeta_t\| \leq \frac{3\|\bg_t^+\|}{1+2\sum_{i=1}^t \|\bg_i\|\|\bg_i^+\|}$ for all $t = 1, \dots,T$.
\end{lemma}

\begin{proof}
By the definition of $\bbeta_{t+1},~\hat{\bbeta}_{t+1},~\bbeta_t$, we have:
\begin{align*}
\| \bbeta_{t+1}-\bbeta_t \| &\leq \| \hat{\bbeta}_{t+1}-\bbeta_t \| = \frac{\|\bg_t^+ + 2\bbeta_t(\|\bg_t\|^2 - \|\bg_t^+ - \bg_t\|^2) \|}{3+2\sum_{i=1}^{t-1} (\|\bg_i\|^2 - \|\bg_i^+ - \bg_i\|^2)}\\
&\leq \frac{\| \bg_t^+ \| + 2\| \bbeta_t\| (2\|\bg_t\| \|\bg_t^+\| - \|\bg_t^+\|^2)}{3+2\sum_{i=1}^{t-1} (2\|\bg_i\|\|\bg_i^+\| - \|\bg_i^+\|^2)} \leq \frac{\| \bg_t^+ \| + 2\| \bbeta_t\| (2\|\bg_t\| \|\bg_t^+\| - \|\bg_t^+\|^2)}{3+2\sum_{i=1}^{t-1} \|\bg_i\|\|\bg_i^+\| } \\
&\leq \frac{\| \bg_t^+ \| + 2\| \bbeta_t\| \|\bg_t^+\|(2 \|\bg_t\|-h_t \|\bg_t\|) }{1+2\sum_{i=1}^{t} \|\bg_i\|\|\bg_i^+\|} \leq \frac{3 \|\bg_t^+\|}{1+2\sum_{i=1}^{t} \|\bg_i\|\|\bg_i^+\|}~,
\end{align*}
where we used the fact that $\|\bg_i\|\|\bg_i^+\|\leq 1$ in second to last inequality.
\end{proof}

\begin{lemma}
\label{lemma:int}  
Let $a_0\geq 0$ and $f :[0,+\infty) \to [0,+\infty)$ a nonincreasing function. Then,
\[
\sum_{t=1}^T a_t f\left(a_0+\sum_{i=1}^t a_i\right)
\leq \int_{a_0}^{\sum_{t=0}^T a_t} f(x)dx~.
\]
\end{lemma}

\begin{proof}
Denote by $s_t = a_0+\sum_{i=1}^t a_i$.
\[
a_t f\left(a_0+ \sum_{i=1}^t a_i\right)
= a_t f(s_t) 
= \int_{s_{t-1}}^{s_t} f(s_t)dx \leq \int_{s_{t-1}}^{s_t} f(x) dx~.
\]
Summing over $t = 1,\dots,T$, we have the stated bound.
\end{proof}

\begin{lemma}
\label{lemma:OGD}
Let $V$ a non-empty closed convex set in $\R^d$. Assume that the functions $f_t: \R \to (-\infty,\infty]$ are $\mu_t$-strongly convex w.r.t $\|\cdot\|_2$ over $V \subset \cap_{t=1}^T \interior \dom f_t$, where $\mu_t>0$. Assume we receive subgradients $\bv_t \in \partial f_t(\bw_t)$ and set $\bw_t$ using Online Gradient Descent with stepsizes: $\eta_t = \frac{1}{1/\eta_1+\sum_{i=1}^{t-1}\mu_i}$. Then, for any $\bu\in V$, we have the following regret guarantee:
\begin{align*}
\sum_{t=1}^T (f_t(\bw_t) - f_t(\bu))
\leq \frac{\|\bw_1 - \bu\|^2}{2\eta_1} + \sum_{t=1}^T \frac{\eta_t \left\| \bv_t\right\|^2}{2}+\sum_{t=1}^T \frac{\mu_t}{2}\|\bw_{t+1}+\bw_t-2\bu\| \|\bw_{t+1}-\bw_t\| ~.
\end{align*}
\end{lemma}

\begin{proof}
From the strongly convexity of the function $f_t$, we have that
\[
f_t(\bw_t)-f_t(\bu)\leq \langle \bv_t, \bw_t -\bu\rangle-\frac{\mu_t}{2}\| \bw_t - \bu\|~,
\]
where $\bv_t \in \partial f_t(\bw_t)$.
From the fact that $\eta_t = \frac{1}{1/\eta_1+\sum_{i=1}^{t-1} \mu_i}$, we have
\[
\frac{1}{2\eta_{t}}+\frac{\mu_t}{2} = \frac{1}{2\eta_{t+1}}, \qquad t=1,\dots,T~.
\]
Also, observe that
\begin{align*}
\|\bw_t - \bu\|^2
&= \|\bw_{t+1} - \bu + \bw_t - \bw_{t+1}\|^2 \\
&= \|\bw_{t+1} - \bu\|^2 + 2 \langle \bw_{t+1}-\bu, \bw_t -\bw_{t+1}\rangle + \|\bw_{t} -\bw_{t+1}\|^2 \\
&= \|\bw_{t+1} - \bu\|^2 + \langle 2\bw_{t+1} - 2\bu + \bw_t-\bw_{t+1}, \bw_t -\bw_{t+1}\rangle \\
&\geq \|\bw_{t+1} - \bu\|^2 - \|\bw_{t+1} + \bw_t - 2\bu\| \|\bw_t -\bw_{t+1}\|~.
\end{align*}
Thus, we have
\begin{align*}
\sum_{t=1}^T (f_t(\bw_t) - f_t(\bu))
&\leq \sum_{t=1}^T \left(\frac{1}{2\eta_t}\|\bw_t - \bu\|^2 - \frac{1}{2\eta_t}\| \bw_{t+1} - \ \bu\|^2 -\frac{\mu_t}{2}\| \bw_t-\bu\|^2 +\frac{\eta_t}{2}\left\| \bv_t\right\|^2 \right)\\
&\leq \sum_{t=1}^T \left(\frac{1}{2\eta_t}\|\bw_t - \bu\|^2 - \frac{1}{2\eta_t}\| \bw_{t+1} - \bu\|^2 -\frac{\mu_t}{2}\| \bw_{t+1}-\bu\|^2 \right.\\
&\qquad \left. + \frac{\mu_t}{2}\|\bw_{t+1} - \bw_t\|\|\bw_{t+1}+\bw_t - 2\bu\|+\frac{\eta_t}{2}\left\| \bv_t\right\|^2 \right)\\
&= \sum_{t=1}^T \left(\frac{1}{2\eta_t}\|\bw_t - \bu\|^2 - \frac{1}{2\eta_{t+1}}\| \bw_{t+1} - \bu\|^2 \right) \\
&\qquad + \sum_{t=1}^T \frac{\mu_t}{2}\|\bw_{t+1} - \bw_t\|\|\bw_{t+1}+\bw_t - 2\bu\| + \sum_{t=1}^T\frac{\eta_t}{2}\left\|\bv_t\right\|^2~.
\end{align*}
Summing and telescoping, we obtain
\[
\sum_{t=1}^T \left( f_t(\bw_t) - f_t(\bu) \right) \leq \frac{\| \bw_1 - \bu\|^2}{2\eta_1} + \sum_{t=1}^T \frac{\mu_t}{2}\|\bw_{t+1} - \bw_t\|\|\bw_{t+1}+\bw_t - 2\bu\| + \sum_{t=1}^T\frac{\eta_t}{2}\left\| \bv_t\right\|^2~. 
\]
\end{proof}

\begin{corollary}
\label{corol:OGD}
Consider an OCO problem with losses $f_t(\bbeta) = \langle \bg_t^+, \bbeta\rangle + \left(\|\bg_t\|^2-\|\bg_t-\bg_t^+\|^2\right)\|\bbeta\|^2$, with $\bbeta\in B$, $B=\{ \bx | \|\bx\|\leq 1/2\}$.

Then, from Lemma~\ref{lemma:OGD}, we have
\[
\sum_{t=1}^T (f_t(\bbeta_t)- f_t(\bbeta^\star))
\leq 3/2 + 17/4\sum_{t=1}^T \frac{2\|\bg_t\| \|\bg_t^+\|}{1+2\sum_{i=1}^t \|\bg_i\| \|\bg_i^+\|}~.
\]
\end{corollary}

\begin{proof}
From Lemma~\ref{lemma:OGD}, 
\begin{align*}
\sum_{t=1}^T (f_t(\bbeta_t) - f_t(\bbeta^\star))
\leq \frac{\| \bbeta_1 - \bbeta^\star\|^2}{2\eta_1} + \sum_{t=1}^T \frac{\mu_t}{2}\|\bbeta_{t+1} - \bbeta_t\|\|\bbeta_{t+1}+\bbeta_t - 2\bbeta^\star\| + \sum_{t=1}^T\frac{\eta_t}{2}\left\| \bv_t\right\|^2~.
\end{align*}
Observe that $\bbeta_t, \bbeta_{t+1},\bbeta^\star \in B$, so $\|\bbeta_{t+1}+\bbeta_t - 2\bbeta^\star\|\leq 2$. Moreover, we have $\mu_t \leq 2\|\bg_t\|$.
So, summing the second term and using Lemma~\ref{lemma:log}, we obtain
\[
\sum_{t=1}^T \frac{\mu_t}{2}\|\bbeta_{t+1} - \bbeta_t\|\|\bbeta_{t+1}+\bbeta_t - 2\bbeta^\star\|
\leq \sum_{t=1}^T \mu_t\|\bbeta_{t+1} - \bbeta_t\|
\leq 3\sum_{t=1}^T \frac{2\|\bg_t\| \|\bg_t^+\|}{1+2\sum_{i=1}^t \|\bg_i\|\|\bg_i^+\|}~.
\]
Noting that $\bv_t \in \partial f_t(\bbeta_t) = \bg_t^+ + 2\bbeta_t(\|\bg_t\|^2 - \|\bg_t^+ - \bg_t\|^2)$, we have
\[
\left\|\bv_t \right\|^2 \leq \|\bg_t^+\|^2 + 4\|\bbeta_t\|\|\bg_t^+\|\|\bg_t\| + 4 \|\bbeta_t\|^2\|\bg_t\|(2\|\bg_t^+\|) 
\leq 5\|\bg_t^+\|\|\bg_t\| ~.
\]
Hence, summing the third term, we have
\[
\sum_{t=1}^T\frac{\eta_t}{2}\left\| \bv_t\right\|^2 
\leq \frac{5}{4}\sum_{t=1}^T \frac{2\|\bg_t\| \|\bg_t^+\|}{1+2\sum_{i=1}^t \|\bg_i\| \|\bg_i^+\|}~.
\]
Now using the fact that $\frac{1}{\eta_1} = 3$, we obtain
\[
\sum_{t=1}^T f_t(\bbeta_t) - f_t(\bbeta^\star)
\leq 3/2+\frac{17}{4}\sum_{t=1}^T \frac{2\|\bg_t\| \|\bg_t^+\|}{1+2\sum_{i=1}^t \|\bg_i\| \|\bg_i^+\|}~. 
\]
\end{proof}

\begin{lemma}
\label{lemma:proj}
Consider an OCO problem with losses $f_t(\bbeta) = \langle \bg_t^+, \bbeta\rangle + \frac{\mu_t}{2}\|\bbeta\|^2$ and \\
$\mu_t=2\left(\|\bg_t\|^2-\|\bg_t-\bg_t^+\|^2\right)$, with $\bbeta\in B=\{ \bx| \|\bx\|\leq r  \}$. Define $\bbeta_T^*:= \arg\min_{\bbeta\in B}\sum_{t=1}^T f_t(\bbeta)$. Then we have 
\[
\sum_{t=1}^T  f_t(\bbeta_T^\star) \leq \max \left\{ \frac{ -r\left\|\sum_{t=1}^{T} \bg_t^+\right\|}{2}, \frac{-\left\|\sum_{t=1}^T \bg_t^+\right\|^2}{2\sum_{t=1}^T \mu_t } \right\}~. 
\]
\end{lemma}

\begin{proof}
We consider 2 cases: By definition $\bbeta^\star = \bbeta_{T+1} = \prod_B (\hat{\bbeta}_{T+1})$, $f_t$ is $\mu_t$ strongly convex, where $\mu_t = 2(\|\bg_t\|^2 - \|\bg_t^+ -\bg_t\|^2)$.
\begin{itemize}
\item If $\hat{\bbeta}_{T+1}\in B$, that is $\bbeta_{T+1} = \hat{\bbeta}_{T+1}=\frac{-\sum_{t=1}^T \bg_t^+}{\sum_{t=1}^T \mu_t}$. In this case,
\[
\sum_{t=1}^T f_t(\bbeta^*) =\left(\sum_{t=1}^T \bg_t^+ \right)\frac{-\sum_{t=1}^T \bg_t^+ }{\sum_{t=1}^T \mu_t}+ \frac{\left\| -\sum_{t=1}^T \bg_t^+ \right\|^2}{\left(\sum_{t=1}^T \mu_t\right)^2} \frac{\sum_{t=1}^T \mu_t}{2} = -\frac{1}{2}\frac{\left\| \sum_{t=1}^T \bg_t^+ \right\|^2}{\sum_{t=1}^T \mu_t}~.
\]

\item If $\hat{\bbeta}_{T+1} \not\in B$, $\bbeta_{T+1} = \prod_B (\hat{\bbeta}_{T+1}) = r\frac{-\sum_{t=1}^T \bg_t^+}{\|-\sum_{t=1}^T \bg_t^+\|} $. In this case, 
\[
\|\hat{\bbeta}_{T+1}\| = \left\|\frac{-\sum_{t=1}^T \bg_{t}^+}{ \sum_{t=1}^T \mu_t}\right\| \geq r~. 
\]
That is, 
\[
\left\|\sum_{t=1}^T \bg_{t}^+ \right\| \geq r\sum_{t=1}^T \mu_t~.
\]
Therefore, we have 
\begin{align*}
\sum_{t=1}^T  f_t(\bbeta^*)&= -r\left\|\sum_{t=1}^T \bg_t^+\right\| + r^2 \sum_{t=1}^T \frac{\mu_t}{2} \left(\frac{\|\sum_{t=1}^T \bg_t^+\| }{\|\sum_{t=1}^T \bg_t^+\|}\right)^2 \leq -r\left\|\sum_{t=1}^T \bg_t^+\right\| +  \frac{r^2}{2} \frac{\left\|\sum_{t=1}^T \bg_t^+ \right\|}{r} \\
&=-\frac{r}{2} \left\|\sum_{t=1}^T \bg_t^+ \right\| ~.
\end{align*}
\end{itemize}

In conclusion, 
\[
\sum_{t=1}^T  f_t(\bbeta^\star) \leq \max \left\{ \frac{ -r\left\|\sum_{t=1}^{T} \bg_t^+\right\|}{2}, \frac{-\left\|\sum_{t=1}^T \bg_t^+\right\|^2}{2\sum_{t=1}^T \mu_t } \right\}~. 
\]
\end{proof}

\begin{lemma}
\label{lemma:one_step}
Let $\|\bg_t\|\leq 1$, $\bg_t^+ = h_t\bg_t$ where $h_t \in [0,1]$, for $t = 1,\dots,T$. Then, we have
\begin{align*}
\ln(1-\langle \bg_t,\bbeta_t \rangle)-\ln(1+\langle \bg_t^+ - \bg_t,\bbeta_{t+1}\rangle) 
&\geq -\langle \bg_t^+,\bbeta_t\rangle - (\|\bg_t\|^2-\| \bg_t^+ - \bg_t\|^2)\|\bbeta_t\|^2 \\
&\quad - 2\|\bg_t\| \|\bbeta_{t+1}-\bbeta_{t}\|~.
\end{align*}
\end{lemma}

\begin{proof}
\begin{align*}
\ln(&1-\langle \bg_t,\bbeta_{t}\rangle) - \ln(1+\langle \bg_t^+-\bg_t,\bbeta_{t+1}\rangle)\\
&\geq \ln(1-\langle \bg_t,\bbeta_{t}\rangle) - \ln(1+\langle \bg_t^+-\bg_t,\bbeta_{t}\rangle + \|\bg_t^+-\bg_t\| \|\bbeta_{t+1}-\bbeta_t\|) \\
& = \ln(1-\langle \bg_t,\bbeta_{t}\rangle) - \ln(1+\langle \bg_t^+-\bg_t,\bbeta_{t}\rangle) - \ln\left( 1+ \frac{\|\bg_t^+-\bg_t\| \|\bbeta_{t+1}-\bbeta_t\|}{1+ \langle \bg_t^+-\bg_t, \bbeta_t\rangle}\right) \\
& \geq \ln(1-\langle \bg_t,\bbeta_{t}\rangle) - \ln(1+\langle \bg_t^+-\bg_t,\bbeta_{t}\rangle) - \frac{\|\bg_t^+-\bg_t\| \|\bbeta_{t+1}-\bbeta_t\|}{1+ \langle \bg_t^+-\bg_t, \bbeta_t\rangle} \\
& \geq \ln(1-\langle \bg_t,\bbeta_{t}\rangle) - \ln(1+\langle \bg_t^+-\bg_t,\bbeta_{t}\rangle) - \frac{\|\bg_t^+-\bg_t\| \|\bbeta_{t+1}-\bbeta_t\|}{1-1/2}\\
&\geq \ln(1-\langle \bg_t,\bbeta_{t}\rangle) - \ln(1+\langle \bg_t^+-\bg_t,\bbeta_{t}\rangle) - 2\|\bg_t\| \|\bbeta_{t+1}-\bbeta_t\|,
\end{align*}
where in the last inequality we used the fact that $\bg_t^+=h_t \bg_t$, $0\leq h_t\leq1$, $\|\bbeta_t\|\leq 1/2$.
The sum of these last terms is upper bounded by a logarithmic term.

Now, considering the first two terms, we have
\[
\ln(1-\langle \bg_t,\bbeta_{t}\rangle) - \ln(1+\langle \bg_t^+-\bg_t,\bbeta_{t}\rangle)
= f(\langle \bg_t,\bbeta_{t} \rangle),
\]
where $f(x)=\ln(1-x)- \ln(1+(h_t-1)x)$ for $|x|\leq 1/2$.

To have the tightest inequality, we consider two cases separately.

\textbf{Case $0\leq x\leq 1/2$.}
We consider $g(x)=f(x)-h_t x$. We have that the first derivative of $g$ is negative for $0\leq x\leq 1/2$.
Moreover, the function $\phi(x)=x^2 (2 h_t -h_t^2)$ is increasing for $0\leq x\leq 1/2$.
Hence, for $0\leq x\leq 1/2$, we have
\[
\frac{\ln(1-x)- \ln(1+(h_t-1)x)-h_t x}{x^2 (2 h_t -h_t^2)}
=\frac{g(x)}{\phi(x)}
\geq \frac{g(1/2)}{\phi(1/2)}
=\frac{-4\ln(1+h_t)-2h_t}{2h_t -h_t^2}
\geq -1,
\]
where in the last inequality we used the fact that $0\leq h_t\leq 1$.

\textbf{Case $-1/2\leq x<0$.}
Here, we lower bound $f$ using a Taylor expansion:
\[
f(x) =  f(0) + x f'(0) + \frac{x^2}{2} f''(y),
\]
where $y$ is between $0$ and $x$.
Denoting $a=h_t-1$, we have
\begin{align*}
f(x) 
= - x h_t - \frac{x^2}{2}  \frac{(a+1)(2ay-a+1)}{(y-1)^2 (a y+1)^2}~.
\end{align*}
Now, consider the quadratic term in the above expression. Dividing and multiplying by $1-a\geq0$, we have
\[
\frac{(1-a^2)(2ay-a+1)}{(1-a)(y-1)^2 (a y+1)^2}~.
\]
We are interested in keeping the term $1-a^2$ and upper bounding the rest. Given that we are considering the case $-1/2\leq x<0$, we have $-1/2\leq y<0$.
So, we have
\[
2ay-a+1 \leq 1-a-y+ay = (1-a)(1-y)~.
\]
and
\begin{align*}
\frac{2ay-a+1}{(1-a)(y-1)^2 (a y+1)^2}&\leq \frac{(1-a)(1-y)}{(1-a)(y-1)^2 (a y+1)^2}=\frac{1}{ (1-y)  (a y+1)^2} \leq 1~.
\end{align*}

Hence, putting all together, we have
\begin{align*}
\ln(1-\langle \bg_t,\bbeta_{t}\rangle) &- \ln(1+\langle \bg_t^+-\bg_t,\bbeta_{t}\rangle)\geq - h_t \langle \bg_t, \bbeta_t\rangle - (1-a^2) (\langle \bg_t, \bbeta_t\rangle)^2 \\
&\geq - h_t \langle \bg_t, \bbeta_t\rangle - (1-a^2) \|\bg_t\|^2 \|\bbeta_t\|^2 \\
&= - \langle \bg^+_t, \bbeta_t\rangle - (\|\bg_t\|^2 - \|\bg_t\|^2(h_t-1)^2)\|\bbeta_t\|^2 \\
&= - \langle \bg^+_t, \bbeta_t\rangle - (\|\bg_t\|^2 - \|\bg_t^+ - \bg_t\|^2 )\|\bbeta_t\|^2~. 
\end{align*}
\end{proof}

Now, we present the full proof of Theorem~\ref{thm:main}. Note that a part of the proof of Theorem~\ref{thm:main_closed_form} is similar to the proof of Theorem~\ref{thm:main}. 
\begin{proof}[Proof of Theorem~\ref{thm:main}]
From the \eqref{eq:wealth}, we get
\[
\Wealth_t = \Wealth_{t-1} - \langle \bg_t, \bw_t-\bw_{t+1}\rangle - \langle \bg_t^+, \bw_{t+1}\rangle~. 
\]
Using the fact that the algorithm predicts with $\bw_t=\bbeta_t\Wealth_{t-1}$, we obtain
\[
\Wealth_t = \Wealth_{t-1}\frac{1-\langle \bg_t,\bbeta_{t}\rangle}{1+\langle \bg_t^+ - \bg_t,\bbeta_{t+1}\rangle}~.
\]
This implies that $\ln \Wealth_T =\ln \epsilon + \sum_{t=1}^T \ln(1-\langle \bg_t, \bbeta_{t}\rangle) - \ln(1+\langle \bg_t^+-\bg_t,\bbeta_{t+1}\rangle)$.

Using Lemma~\ref{lemma:one_step} and Lemma~\ref{lemma:log}, we have
\begin{align*}
\sum_{t=1}^T &\left(\ln(1-\langle \bg_t,\bbeta_{t}\rangle) - \ln(1+\langle \bg_t^+-\bg_t,\bbeta_{t+1}\rangle) \right)\\
&\geq \sum_{t=1}^T \left[-\langle \bg^+_t, \bbeta_t\rangle -  (\|\bg_t\|^2 - \|\bg_t^+ - \bg_t\|^2 )\|\bbeta_t\|^2 -2\|\bg_t\| \|\bbeta_{t+1}-\bbeta_t\|\right]\\
&\geq \sum_{t=1}^T \Bigg[- \langle \bg^+_t, \bbeta_t\rangle - 
 (\|\bg_t\|^2 - \|\bg_t^+ - \bg_t\|^2 )\|\bbeta_t\|^2 -\frac{6\|\bg_t\|\| \bg_t^+\|}{1+2\sum_{i=1}^t \|\bg_i\|\| \bg_i^+\|}\Bigg]\\
&\geq - 3 \ln\left(1+2\sum_{t=1}^T \|\bg_t\|\| \bg_t^+\|\right) -\sum_{t=1}^T f_t(\bbeta_t)~.
\end{align*}
Where $f_t$ is as defined in Corollary~\ref{corol:OGD}. Applying Corollary, and defining $\mu_t$ as in Lemma~\ref{lemma:proj}, we have 
\begin{align*}
&\geq -3/2 - 7.25\ln\left(1+2\sum_{t=1}^T \|\bg_t\|\| \bg_t^+\|\right) - \sum_{t=1}^T f_t(\bbeta^\star)\\
&\geq -3/2-7.25\ln \left(1+2\sum_{t=1}^T \|\bg_t\|\| \bg_t^+\|  \right)-\max \left\{\frac{-\left\|\sum_{t=1}^{T} \bg_t^+\right\|}{4}, \frac{-\left\|\sum_{t=1}^T \bg_t^+\right\|^2}{2\sum_{t=1}^T \mu_t} \right\}\\
&= -3/2 -7.25\ln \left(1+2\sum_{t=1}^T \|\bg_t\|\| \bg_t^+\| \right) + \min\left\{\frac{\left\|\sum_{t=1}^{T} \bg_t^+\right\|}{4}, \frac{\left\|\sum_{t=1}^T \bg_t^+\right\|^2}{2\sum_{t=1}^T \mu_t} \right\}~.
\end{align*} 
The last inequality is from Lemma~\ref{lemma:proj}.
Next, we perform a case analysis to derive the upper bound for $\Regret_T$.

\textbf{If $\min \left\{\frac{\left\|\sum_{t=1}^{T} \bg_t^+\right\|}{4}, \frac{\left\|\sum_{t=1}^T \bg_t^+\right\|^2}{2\sum_{t=1}^T \mu_t} \right\} = \frac{\left\|\sum_{t=1}^{T} \bg_t^+\right\|}{4}$}, then
\[
\Wealth_T 
\geq \frac{ e^{-3/2}}{\left(1+2\sum_{t=1}^T  \|\bg_t\|\| \bg_t^+\|  \right)^{7.25}} \exp\frac{\left\|\sum_{t=1}^{T} \bg_t^+\right\|}{4}~.
\]
By \citet[Lemma 18]{CutkoskyO18}, we have
\begin{align*}
\Regret_T(\bu) \leq 4\|\bu\| \left(\ln\frac{4\|\bu\|\left(1+2\sum_{t=1}^T  \|\bg_t\|\| \bg_t^+\|  \right)^{7.25}}{e^{-3/2}}-1\right)~.
\end{align*}

\textbf{If $\min \left\{\frac{\left\|\sum_{t=1}^{T} \bg_t^+\right\|}{4}, \frac{\left\|\sum_{t=1}^T \bg_t^+\right\|^2}{2\sum_{t=1}^T \mu_t} \right\} = \frac{\left\|\sum_{t=1}^T \bg_t^+\right\|^2}{2\sum_{t=1}^T \mu_t}$ }, then
\[
\Wealth_T 
\geq \frac{e^{-3/2}}{\left(1+2\sum_{t=1}^T \|\bg_t\|\| \bg_t^+\|  \right)^{7.25}} \exp\frac{\left\|\sum_{t=1}^T \bg_t^+\right\|^2}{2\sum_{t=1}^T \mu_t}~.
\]
By \citet[Lemma 1]{OrabonaT17}, we have
\begin{align*}
&\Regret_T(\bu) \leq - \frac{e^{-3/2}}{\left(1+2\sum_{t=1}^T  \|\bg_t\|\| \bg_t^+\|  \right)^{7.25}}\\
 &+\|\bu\| \sqrt{2\sum_{t=1}^T (2 \|\bg_t\|\| \bg_t^+\|  - \|\bg_t^+\|^2)\ln\left(1+ \frac{\left(\sum_{t=1}^T \mu_t \right) \|\bu\|^2 \left(1+2\sum_{t=1}^T  \|\bg_t\|\| \bg_t^+\|  \right)^{14.5}}{e^{-3}} \right)}~.
\end{align*}

Finally, combining the two cases result gives the regret bound
\begin{align*}
\Regret_T(\bu) &= \mathcal{O}\left(  \max\left\{ \|\bu\| \ln\left(\|\bu\|\left(1+\sum_{t=1}^T \|\bg_t\|\|\bg_t^+\| \right)\right)\right.\right. ,\\
& \left.\left.\|\bu\|\sqrt{\sum_{t=1}^T (2\|\bg_t\|\|\bg_t^+\|-\|\bg_t^+\|^2) \cdot \ln\left(1+\|\bu\|\sum_{t=1}^T (2\|\bg_t\|\|\bg_t^+\|-\|\bg_t^+\|^2)\right)} \right\}  \right). 
\end{align*}
\end{proof}

\section{Proof of Theorem~\ref{thm:main_closed_form}}

In this section, we present the proofs of Theorem~\ref{thm:main_closed_form} and the Lemmas required for its proof.

Lemma~\ref{lemma:f_h_rel} shows that the regret of $\phi_t(\bbeta_t)$ upper bounds the regret of $f_t(\bbeta_t)$, and Lemma~\ref{lemma:no_proj_OGD} gives the upper bound of the regret of running $\phi_t(\bbeta_t)$. Thus we are able to obtain similar results as in Corollary~\ref{corol:OGD}. Lemma~\ref{lemma:no_proj} proves that Algorithm~\ref{alg:imp_coin_closed_form} guarantees $\|\bbeta_t\|\leq 1/2$ which removes the projection step, and leads to the closed form updates.
\begin{lemma}
\label{lemma:no_proj_OGD}
For $t = 1,\dots, T$, $\bbeta_t$ are outputs of running OGD with stepsizes: $\eta_t = \frac{1}{1/\eta_1+\sum_{i=1}^{t-1}\mu_i}$ on the strongly convex losses $\phi_t(\bbeta)$  defined as following:
\begin{itemize}
\item If $\| \bbeta_t\| < \frac{3}{8}$, $\phi_t(\bbeta)=f_t(\bbeta)$,
\item If $\frac{3}{8} \leq \|\bbeta_t\| \leq \frac{1}{2}$, $\phi_t(\bbeta) = C\|\bg_t^+\| \|\bbeta\|^2$,
\end{itemize}
where $\bbeta_t\in B$ and $B = \{\bx : \|\bx\|\leq \frac{1}{2}\}$. $\phi_t(\bbeta)$ is $\mu_t$ strongly convex. $\bbeta_1^\star \triangleq \arg\min_{\bbeta\in B_1}\sum_{t=1}^T f_t(\bbeta)$, where $B_1 = \{\bx : \|\bx\|\leq \frac{1}{4}\}$. Then, for $C \geq 3$ we have
\begin{align*}
\sum_{t=1}^T &(\phi_t(\bbeta_t)- \phi_t(\bbeta_1^\star))\leq - \frac{5C^2}{4}  \ln(1/\eta_1-2) + \frac{1}{2\eta_1}+ \frac{5 C^2}{4}\ln\left(1/\eta_1-2+2\sum_{t=1}^{T}\|\bg_t\| \|\bg_t^+\|\right)~.
\end{align*}
\end{lemma}

\begin{proof}
Applying Lemma~\ref{lemma:OGD}, we have
\begin{align*}
\sum_{t=1}^T (\phi_t(\bbeta_t) - \phi_t(\bbeta_1^\star))
\leq \frac{\|\bbeta_1 - \bbeta_1^\star\|^2}{2\eta_1} + \sum_{t=1}^T \left(\mu_t \|\bbeta_{t+1}-\bbeta_t\| + \frac{\eta_t \left\| \nabla \phi_t(\bbeta_t)\right\| ^2}{2}\right)~.
\end{align*}

We now consider the two possible cases for $h_t$.

\textbf{If} $\phi_t(\bbeta) = f_t(\bbeta)$, then 
\[
\left\| \nabla \phi_t(\bbeta_t) \right\| 
= \| \bg_t^+ + 2\bbeta_t(\|\bg_t\|^2 - \| \bg_t^+ - \bg_t\|^2) \| 
\leq 3\|\bg_t^+\|~.
\]
Hence, $\left\| \nabla \phi_t(\bbeta_t) \right\|^2 \leq 9\|\bg_t^+\|^2 \leq 9\|\bg_t\| \|\bg_t^+\|$.
Moreover, $\mu_t = 2(\|\bg_t\|^2 - \|\bg_t^+-\bg_t\|^2)\geq 2\|\bg_t\| \|\bg_t^+\|$.

\textbf{If} $\phi_t(\bbeta)=C \|\bg_t^+\| \|\bbeta_t\|^2$, then
\[
\left\| \nabla \phi_t(\bbeta_t)\right\| = \left\| 2 C\bbeta_t \|\bg_t^+\| \right\| \leq C \|\bg_t^+\|~.
\]
Hence, $\left\| \nabla \phi_t(\bbeta_t) \right\|^2 = C^2 \|\bg_t^+\|^2 \leq C^2 \|\bg_t\| \|\bg_t^+\|$.
Moreover, assuming $C\geq 1$, we have $\mu_t = 2 C\|\bg_t^+\| \geq 2\|\bg_t\| \|\bg_t^+\|$.

Consequently, with $C\geq 3$, we have
\[
\sum_{t=1}^T \frac{\eta_t \left\| \nabla \phi_t(\bbeta_t) \right\|^2}{2} 
\leq \frac{1}{2}\sum_{t=1}^T \frac{C^2 \|\bg_t\| \|\bg_t^+\|}{1/\eta_1+\sum_{i=1}^{t-1}\mu_i} 
\leq \frac{C^2}{4}\sum_{t=1}^T \frac{2 \|\bg_t\| \|\bg_t^+\|}{1/\eta_1+2\sum_{i=1}^{t-1}\|\bg_i\| \|\bg_i^+\|}~.
\]

Applying Lemma~\ref{lemma:no_proj_log} and $C\geq 3$, we thus obtain
\begin{align*}
\sum_{t=1}^T \phi_t(\bbeta_t) &- \phi_t(\bbeta_1^\star) 
\leq \frac{1}{2\eta_1} + \left(\max\left(C^2,3\right)+\frac{C^2}{4}\right) \sum_{t=1}^T \frac{2\|\bg_t\| \|\bg_t^+\|}{1/\eta_1+2\sum_{i=1}^{t-1}\|\bg_i\| \|\bg_i^+\|} \\
& \leq \frac{1}{2\eta_1} + \left(\max\left(C^2,3\right)+\frac{C^2}{4}\right) \sum_{t=1}^T \frac{2\|\bg_t\| \|\bg_t^+\|}{1/\eta_1 - 2 +2\sum_{i=1}^{t}\|\bg_i\| \|\bg_i^+\|}\\
&\leq \frac{1}{2\eta_1} + \left(\max\left(C^2,3\right)+\frac{C^2}{4}\right) \left(\ln\left(1/\eta_1-2+2\sum_{t=1}^{T}\|\bg_t\| \|\bg_t^+\|\right) - \ln(1/\eta_1-2) \right)\\
&=  \frac{5 C^2}{4}\ln\left(1 / \eta_1 -2 +2\sum_{t=1}^{T}\|\bg_t\| \|\bg_t^+\|\right) + \frac{1}{2\eta_1} - \frac{5 C^2}{4}\ln(1/\eta_1 -2)~. 
\end{align*}
\end{proof}

\begin{lemma}
\label{lemma:f_h_rel}
Under conditions and notations of Lemma~\ref{lemma:no_proj_OGD} and for any $C\geq 9$, we have
\[
f_t(\bbeta_t) - f_t(\bbeta_1^\star) \leq \phi_t(\bbeta_t) - \phi_t(\bbeta_1^\star)~.
\]

\end{lemma}

\begin{proof}
We consider the two cases for the losses $h_t$.

When $ \frac{3}{8}\leq\|\bbeta_t\| \leq \frac{1}{2}$, we have $\phi_t(\bbeta_t) =  C \| \bg_t^+\| \|\bbeta_t\|^2$.
Hence, we have
\[
\phi_t(\bbeta_t) - \phi_t(\bbeta_1^\star) 
= C \| \bg_t^+\| \left(\|\bbeta_t\|^2 - \|\bbeta^\star_1\|^2\right) ~.
%\geq  16 \| \bg_t^+\|\left(\left(\frac{3}{8}\right)^2 - \left(\frac{1}{4}\right)^2\right)
%= \frac{5}{4}\| \bg_t^+\|~.
\]
Moreover,
\begin{align*}
f_t(\bbeta_t) - f_t(\bbeta_1^\star) 
&= \langle\bg_t^+,\bbeta_t -\bbeta^\star_1 \rangle  + (2\|\bg_t\| \|\bg_t^+\| - \|\bg_t^+\|^2) \|\bbeta_t\|^2 -   (2\|\bg_t\| \|\bg_t^+\| - \|\bg_t^+\|^2) \|\bbeta^\star_1\|^2\\
&\leq \| \bg_t^+\| (\| \bbeta_t\| + \|\bbeta^\star_1\|) +  \|\bg_t^+\|\|\bg_t\| (2-h_t) (\|\bbeta_t\|^2 - \|\bbeta^\star_1\|^2)\\
&\leq \| \bg_t^+\| (\| \bbeta_t\| + \|\bbeta^\star_1\|) +  \|\bg_t^+\|(\|\bbeta_t\|^2- \|\bbeta^\star_1\|^2)~.
\end{align*}
Hence, we have $f_t(\bbeta_t) - f_t(\bbeta^\star_1) \leq \phi_t(\bbeta_t) - \phi_t(\bbeta^\star_1)$ iff
\[
\| \bbeta_t\| + \|\bbeta^\star_1\| + \|\bbeta_t\|^2 - \|\bbeta^\star_1\|^2\leq C \left(\|\bbeta_t\|^2 - \|\bbeta^\star_1\|^2\right),
\]
that is
\[
C
\geq \frac{\| \bbeta_t\| + \|\bbeta^\star_1\| + \|\bbeta_t\|^2-\|\bbeta^\star_1\|^2}{\|\bbeta_t\|^2 - \|\bbeta^\star_1\|^2}
= \frac{1}{\|\bbeta_t\| - \|\bbeta^\star_1\|}+1~.
\]
Using the fact that $\frac{1}{\|\bbeta_t\| - \|\bbeta^\star_1\|}+1\leq 9$, gives the stated value for C.
When $\|\bbeta_t\| \leq \frac{3}{8}$, we have $\phi_t(\bbeta) = f_t(\bbeta)$ and $f_t(\bbeta_t) - f_t(\bbeta^\star_1) \leq \phi_t(\bbeta_t) - \phi_t(\bbeta^\star_1)$ is trivially true. 
\end{proof}

\begin{lemma}
\label{lemma:no_proj}
Under the assumptions of Theorem~\ref{thm:main_closed_form} and the notation of Algorithm~\ref{alg:imp_coin_closed_form}, let $1/\eta_1\geq \max(2 C, 16)$. Then, $\|\bbeta_t\| \leq \frac{1}{2}$ for all $t = 1, \dots,T$.
\end{lemma}

\begin{proof}
We prove this by induction: \textbf{Base case}: $\bbeta_0 = \boldsymbol{0}$, $\|\bbeta_0\| \leq \frac{1}{2}$ is trivially true. \\
\textbf{Induction step}: Suppose $\|\bbeta_t\|\leq \frac{1}{2}$, we prove $\|\bbeta_{t+1}\| \leq \frac{1}{2}$. We perform the following case analysis:

If $\frac{3}{8}\leq\|\bbeta_t\| \leq \frac{1}{2}$, 
\[
\|\bbeta_{t+1}\| 
= \left\| \bbeta_t - \frac{2 C \bbeta_t \|\bg_t^+\|}{1/\eta_1+\sum_{i=1}^{t-1}\mu_i} \right\|
\leq \|\bbeta_t\| \left(1-\frac{2 C\|\bg_t^+\|}{1/\eta_1+\sum_{i=1}^{t-1} \mu_i}\right)
\leq \|\bbeta_t\|\leq \frac{1}{2}~.
\] 
If $\|\bbeta_t\| < \frac{3}{8}$
\begin{align*}
\| \bbeta_{t+1}\| &= \left\| \bbeta_t - \frac{\bg_t^+ + 2\bbeta_t(2\|\bg_t\|\|\bg_t^+\| - \|\bg_t^+\|^2)}{1/\eta_1+\sum_{i=1}^{t-1}\mu_i} \right\| \leq \| \bbeta_t\| + \left\| \frac{\bg_t^+ + 2\bbeta_t(2\|\bg_t\|\|\bg_t^+\| - \|\bg_t^+\|^2)}{1/\eta_1+\sum_{i=1}^{t-1}\mu_i} \right\| \\
&\leq  \frac{3}{8}+ 2 \eta_1 \leq \frac{1}{2} ~.
\end{align*}
\end{proof}

\begin{lemma}
\label{lemma:no_proj_log}
Under the assumptions of Theorem~\ref{thm:main_closed_form} and the notation of Algorithm~\ref{alg:imp_coin_closed_form}, we have, for any $C\geq 1$, we have 
\[
\| \bbeta_{t+1}-\bbeta_{t}\|
\leq \max\left(C,3\right) \frac{\|\bg_t^+\|}{1/\eta_1+2\sum_{i=1}^{t-1} \|\bg_i\| \|\bg_i^+\|}~.
\]
and
\[
\mu_t \| \bbeta_{t+1}-\bbeta_{t}\|
\leq \max\left(C^2,3\right) \frac{2 \|\bg_t\| \|\bg_t^+\|}{1/\eta_1+2\sum_{i=1}^{t-1} \|\bg_i\| \|\bg_i^+\|}~.
\]
\end{lemma}
\begin{proof}
We consider the two cases for $h_t$.
\begin{itemize}
\item If $\phi_t(\bbeta) = C\|\bg_t^+\| \|\bbeta\|^2$, we have
$\mu_t = 2 C \|\bg_t^+\|$. Assuming $C \geq 1$ we have $2 C \|\bg_t^+\| \geq 2\|\bg_t\| \|\bg_t^+\|$.
Moreover, $\left\| \nabla \phi_t(\bbeta_t) \right\| = 2 C\| \bg_t^+\| \| \bbeta_t\| \leq C\| \bg_t^+\|$.
So, we have
\[
\mu_t \|\bbeta_{t+1}-\bbeta_t \| 
\leq 2 C \|\bg_t^+\| \frac{C\|\bg_t^+\|}{1/\eta_1+\sum_{i=1}^{t-1} \mu_i}
\leq C^2 \frac{2\|\bg_t^+\|\|\bg_t\|}{1/\eta_1+2\sum_{i=1}^{t-1} \|\bg_i\| \|\bg_i^+\|} ~.
%\leq \frac{C^2}{4} \frac{4\|\bg_t^+\|\|\bg_t\|}{60 +4\sum_{i=1}^{t} \|\bg_i\| \|\bg_i^+\|}~.
\]
\item If $\phi_t(\bbeta) = f_t(\bbeta)$, we have
\[
2 \|\bg_t\|
\geq 2 \|\bg_t\|^2 
\geq \mu_t = 2(\|\bg_t\|^2 - \|\bg_t^+ - \bg_t\|^2) 
= 2\|\bg_t\| \|\bg_t^+\| (2-h_t)
\geq 2\|\bg_t\| \|\bg_t^+\|~.
\]
Moreover, $\left\| \nabla \phi_t(\bbeta_t) \right\| \leq 3\|\bg_t^+\|$. 
So, we obtain $\mu_t \|\bbeta_{t+1}-\bbeta_t \| \leq 3 \frac{2\|\bg_t^+\|\|\bg_t\|}{1/\eta_1+2\sum_{i=1}^{t-1} \|\bg_i\| \|\bg_i^+\|}$ .

\end{itemize}
Taking the maximum of the two cases, we obtain the stated upper bounds.

\end{proof}

We can now present the proof of Theorem~\ref{thm:main_closed_form}.
\begin{proof}[Proof of Theorem~\ref{thm:main_closed_form}]
Using Lemma~\ref{lemma:one_step} in the first inequality, and Lemma~\ref{lemma:no_proj_log} in the second inequality, we have
\begin{align*}
&\sum_{t=1}^T \left(\ln(1-\langle \bg_t,\bbeta_{t}\rangle) - \ln(1+\langle \bg_t^+-\bg_t,\bbeta_{t+1}\rangle)\right) \\
&\geq \sum_{t=1}^T \left(- \langle \bg^+_t, \bbeta_t\rangle -  (\|\bg_t\|^2 - \|\bg_t^+ - \bg_t\|^2 )\|\bbeta_t\|^2 -2\|\bg_t\| \|\bbeta_{t+1}-\bbeta_t\|\right)\\
&\geq \sum_{t=1}^T \left(- \langle \bg^+_t, \bbeta_t\rangle -  (\|\bg_t\|^2 - \|\bg_t^+ - \bg_t\|^2 )\|\bbeta_t\|^2 -C\frac{2\|\bg_t\|\| \bg_t^+\|}{1/\eta_1-2+2\sum_{i=1}^t \|\bg_i\|\|\bg_i^+\|}\right)\\
&\geq - C\ln\left(1/\eta_1-2+2\sum_{t=1}^T \|\bg_t\|\|\bg_t^+\|\right) -\sum_{t=1}^T f_t(\bbeta_t)\\
&\geq -\frac{1}{2\eta_1} + \frac{5C^2}{4}\ln(1/\eta_1 -2) - \left(\frac{5C^2}{4}+C\right) \ln\left(1/\eta_1 -2+2\sum_{t=1}^T \|\bg_t\|\|\bg_t^+\|\right) - \sum_{t=1}^T f_t(\bbeta_1^\star)\\
&\geq -\left(\frac{5C^2}{4}+C\right) \ln\left(1/\eta_1 -2+2\sum_{t=1}^T \|\bg_t\|\|\bg_t^+\|\right) - \sum_{t=1}^T f_t(\bbeta_1^\star)\\
& = -110.25 \ln\left(16+2\sum_{t=1}^T \|\bg_t\|\|\bg_t^+\|\right) - \sum_{t=1}^T f_t(\bbeta_1^\star)~,
\end{align*}
The second to the last inequality is from Lemma~\ref{lemma:f_h_rel} and Lemma~\ref{lemma:no_proj_OGD}. $C=9$ gives rise to the last inequality and the last equation. 
We now use Lemma~\ref{lemma:proj} to obtain
\[
- \sum_{t=1}^T f_t(\bbeta_1^\star)
\geq -\max \left\{\frac{-\left\|\sum_{t=1}^{T}\bg_t^+\right\|}{8}, \frac{-\left\|\sum_{t=1}^T \bg_t^+\right\|^2}{2\sum_{t=1}^T \mu_t} \right\}\\
= \min \left\{\frac{\left\|\sum_{t=1}^{T}\bg_t^+\right\|}{8}, \frac{\left\|\sum_{t=1}^T \bg_t^+\right\|^2}{2\sum_{t=1}^T \mu_t} \right\}~.
\]
The rest proof is similar to the proof of Theorem~\ref{thm:main}. Finally, we obtain the same bound as in Theorem~\ref{thm:main}, up to constants hidden in big O notation.

\end{proof}

\section{Proof of Theorem~\ref{thm:coor}}

\begin{proof}
In each coordinate, we perform the regret decomposition as following:
\begin{align*}
\Regret&_T(\bu) =\sum_{t=1}^T \ell_t(\bw_t) - \ell_t(\bu) = \sum_{t=1}^T \hat{\ell}_t(\bw_t) - \hat{\ell}_t(\bw_{t+1})+\hat{\ell}_t(\bw_{t+1})-\hat{\ell}_t(\bu)\\
& \leq \sum_{t=1}^T \langle \bg_t,\bw_t - \bw_{t+1} \rangle + \langle \bg_t^+ , \bw_{t+1}-\bu \rangle = \sum_{t=1}^T \sum_{i=1}^d g_{t,i}(w_{t,i }- w_{t+1,i}) + g_{t,i}^+(w_{t+1,i}-u_i)\\
&= \sum_{i=1}^d \sum_{t=1}^T g_{t,i}(w_{t,i }- w_{t+1,i}) + g_{t,i}^+(w_{t+1,i}-u_i) \\
&= \sum_{i=1}^d  u_i\left( -\sum_{t=1}^T g_{t,i}^+ \right)  - \left( - \sum_{t=1}^T  (g_{t,i}(w_{t,i} - w_{t+1,i})  + g_{t,i}^+ w_{t+1,i}) \right)~,
\end{align*}
where $\bg_t \in \partial \ell_t(\bw_t)$, $\bg_t \in \partial \hat{\ell}_t(\bw_t)$, $ \bg_t^+ \in \partial \hat{\ell}_t(\bw_{t+1})$.

Define $\Wealth_{T,i} = \epsilon - \sum_{t=1}^T  (g_{t,i}(w_{t,i} - w_{t+1,i})  + g_{t,i}^+ w_{t+1,i})$. Suppose we obtain a bound $\Wealth_{T,i}\geq \psi_T\left( \sum_{i=1}^{T} g_{t,i}^+\right)$ for some $\psi_T$. Using the definition of Fenchel conjugate, and the lower bound on the wealth in each coordinate, we have
\begin{align*}
\Regret_T(\bu)  &\leq \sum_{i=1}^d  u_i\left( -\sum_{t=1}^T g_{t,i}^+ \right) - (\Wealth_{T,i}-\epsilon) 
\leq  \sum_{i=1}^d \epsilon + u_i\left( -\sum_{t=1}^T g_{t,i}^+ \right) - \psi_T\left( \sum_{i=1}^{T} g_{t,i}^+\right)\\
&\leq \sum_{i=1}^d \epsilon + \sup_{y} y\cdot u_i - \psi_T(y) = \sum_{i=1}^d \epsilon + \psi^*_T(u_i)~.
\end{align*}
In each coordinate, Algorithm~\ref{alg:imp_coin_coor} is a specific case of running Algorithm~\ref{alg:imp_coin_closed_form} with $d=1$. Applying the result in Theorem~\ref{thm:main_closed_form}, we obtain the final bound. 
\end{proof}

\section{List of Datasets}

In our empirical evaluation, we used 3 regression datasets and 3 classification datasets from the LIBSVM website~\citep{ChangL01} and OpenML~\citep{VanschorenVBT2013}, randomly selected among the ones with a large number of samples. For the OpenML datasets, categorical features are one-hot-encoded. A short summary of the datasets is in Table~\ref{table:dataset}.

\begin{table}[h]
\caption{Datasets in experiments.}
\label{table:dataset}
\begin{center}
\begin{small}
\begin{sc}
\begin{tabular}{lcccr}
\toprule
Dataset & Type & Number of samples & Number of features \\
\midrule
CPU-act   & classification & 8192  & 21 \\
2dPlane   & classification & 40768  & 10\\
Houses   & classification & 20640  & 8\\
Rainfall& regression     & 16755  & 3\\
Bank32nh& regression     & 8192  & 32\\
Houses-8L& regression     &22784  & 8\\
\bottomrule
\end{tabular}
\end{sc}
\end{small}
\end{center}
\end{table}